\documentclass[10pt,twocolumn,letterpaper]{article}
\usepackage{silence}
\usepackage{cvpr}
\usepackage{times}
\usepackage{epsfig}
\usepackage{graphicx}
\usepackage{amsmath}
\usepackage{amssymb}
\usepackage{cite}
\usepackage{dsfont}
\usepackage[colorlinks,linkcolor=blue]{hyperref}
\usepackage[nameinlink,capitalise]{cleveref}

\usepackage{float}
\usepackage{amsthm}
\usepackage{booktabs}
\usepackage{multirow}
\usepackage{siunitx}
\usepackage{tabularx,multirow,booktabs,blindtext}
\newtheorem{thm}{Theorem}

\usepackage{multirow}

\def\cvprPaperID{4931} % *** Enter the CVPR Paper ID here

\ifcvprfinal\fi
\begin{document}

%%%%%%%%% TITLE
\title{Multiview Cross-supervision for Semantic Segmentation}

\author{Yuan Yao\\
University of Minnesota\\
{\tt\small yaoxx340@umn.edu}
\and
Hyun Soo Park\\
University of Minnesota\\
{\tt\small hspark@umn.edu}
}

\maketitle
%\thispagestyle{empty}

%%%%%%%%% ABSTRACT
\begin{abstract}
This paper presents a semi-supervised learning framework for a customized semantic segmentation task using multiview image streams. A key challenge of the customized task lies in the limited accessibility of the labeled data due to the requirement of prohibitive manual annotation effort. We hypothesize that it is possible to leverage multiview image streams that are linked through the underlying 3D geometry, which can provide an additional supervisionary signal to train a segmentation model. We formulate a new cross-supervision method using a shape belief transfer---the segmentation belief in one image is used to predict that of the other image through epipolar geometry analogous to shape-from-silhouette. The shape belief transfer provides the upper and lower bounds of the segmentation for the unlabeled data where its gap approaches asymptotically to zero as the number of the labeled views increases. We integrate this theory to design a novel network that is agnostic to camera calibration, network model, and semantic category and bypasses the intermediate process of suboptimal 3D reconstruction. We validate this network by recognizing a customized semantic category per pixel from realworld visual data including non-human species and a subject of interest in social videos where attaining large-scale annotation data is infeasible. 
\end{abstract}

%%%%%%%%% BODY TEXT
\section{Introduction}
In aid of large-scale visual data, convolutional neural networks (CNN) have been transforming the level of understanding of pixels, which allows deep reasoning about their spatial extent and semantic meaning (e.g., human, bicycle, and horse)~\cite{Long:2017,deeplabv3plus2018,he:2017,chen:2017}. Looking ahead, these models are expected to solve various unprecedented visual tasks customized for our personal data (e.g., recognizing pixels of my daughter among her classmates from a collection of her school play photos). However, such task customization is fundamentally limited by the ability to access the training labels for the personal data. Existing semantic segmentation approaches are mostly built upon the per-pixel semantic label manually annotated by thousands of the crowd workers such as MS COCO~\cite{lin:2014} that constitutes 2.5 millions of segmentation instances. Unfortunately, attaining such large annotations for the customized segmentation task is often infeasible, which introduces a large bias in the trained model because the required number of the training data is known to be equivalent to that of the perceptrons~\cite{zhu:2015}.

In the meantime, as a small form factor of cameras accelerates a seamless integration into our daily lives, now many scenes are recorded by multiple cameras (e.g., Amazon Cloud Cam and Nest Cam), and they will permeate more and deeper. Notably, there is an emerging trend of social videos~\cite{arev:2014,park:2012,fathi:2012,tali:2012}---a collection of videos that record an activity of interest (e.g., political rally, concert, and wedding) from social members at the same time\footnote{There exist multiple online repositories such as Rashomon Project~\cite{rashomon} and CrowdSync cellphone app~\cite{crowd} that host the social videos.}. These cameras readily produce terascale multiview image streams, which opens up a new opportunity to address the annotation challenge for a customized task. In this paper, we formulate a new multiview theory for semi-supervised semantic segmentation to train a CNN from the limited number of the labeled data ($<$15\%) by leveraging the multiview image streams.

\begin{figure}[t!]
  \begin{center}
    \includegraphics[width=0.47\textwidth]{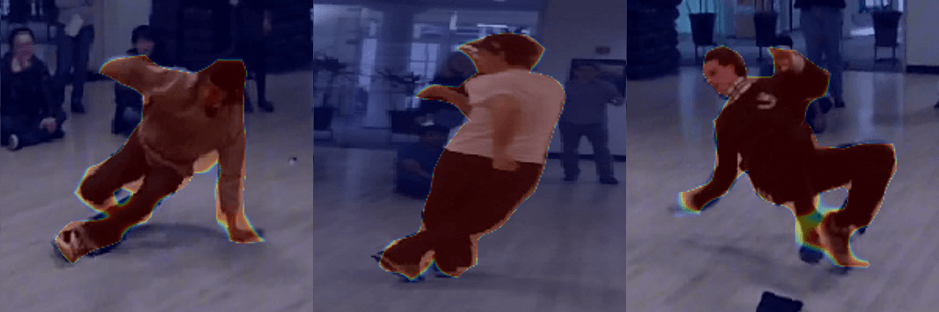}
  \end{center}
     \vspace{-3mm}
  \caption{We design a semi-supervised learning to train a semantic segmentation model using multiview cross-supervision based on shape transfer. This enables customizing a semantic segmentation task, e.g., a b-boy dancer segmentation from social cameras.
%   \yasamin{MONET Pose Estimation on Monkey subjects. MONET trains a semi-supervised network by minimizing epipolar divergence from multiview RGB images.}
  }
  \label{Fig:intro}
  %\vspace{3mm}
\end{figure}

A key property of multiview images is that they are linked through the underlying 3D geometry, which can be beneficial for training a segmentation model. However, the representations used for 3D reconstruction and CNNs often mismatch, i.e., vector vs. raster representations, which makes a tight integration of 3D geometric knowledge into the process of the network training challenging. Existing methods take either a) the approach that alternates between offline 3D reconstruction and training~\cite{simon:2017,Vijayanarasimhan:2017,Byravan:2016,bertasius:2016_unsupervised}; or b) the approach that predicts the 3D geometry from a single view image with additional depth supervision~\cite{kanazawaHMR18,drcTulsiani17,zhou2017unsupervised}. The main limitation of these approaches is that their performance is in principle bounded by the reconstruction quality, which is often suboptimal.

Instead, we present a new multiview learning theory for a customized semantic segmentation task that integrates 3D geometry into the process of segmentation model training, which bypasses the intermediate reconstruction. We introduce a {\em shape belief transfer}---the segmentation belief in one image is used to predict that of the other image through epipolar geometry. We formulate this shape belief transfer as an inverse problem of shape-from-silhouette~\cite{Laurentini:1994,Laurentini:1997,Guan:2006,djelouah:hal-01115557} that reconstructs a 3D object volume (visual hull) from the foreground segmentation of multiview images~\cite{matusik:2000,Kutulakos:2000}. The shape belief transfer is a composition of two belief transfers: (a) 3D shape reconstruction by triangulating the segmentation probability in multiview source images; and (b) 2D projection of the reconstructed 3D shape onto a target view to approximate its segmentation probability. We derive that these two transfers can be combined in a differentiable fashion, and therefore, the end-to-end training is possible. This allows relating the segmentation across views where the unlabeled data can be cross-supervised by the labeled data.

A new theory of the shape belief transfer is derived, which provides the upper and lower bounds of the segmentation for the unlabeled data where its gap approaches asymptotically to zero as the number of the labeled views increases. We further show that the shape belief transfer can be implemented by incorporating stereo rectification that transforms the operation of 2D projection into max-pooling operation to gain significant computational efficiency. Based on the theory, we design a triplet network that takes as input multiview image streams with the limited number of the labeled data and outputs a semantic segmentation model that can reliably predict on the unlabeled data as shown in \Cref{Fig:intro}. The network is trained by minimizing the geometric inconsistency of multiview segmentation, resulting in multiview cross-supervision.

% In this paper, we design a triple network which takes as input multiview image streams with the limited number of the labeled data and outputs a semantic segmentation model that can reliably predict on the unlabeled data. 

% We construct segmentation upper-bound constraints based on visual hull for unlabeled data using all labeled views as a segmentation prior. The network minimizes the error between the output and annotations in its first branch, and at the same time, its additional label-agnostic triplet branches learn to reduce upper-bound violations of any data prediction, which is the visual hull constructed using the predictions of any other two views. 

This framework is flexible: (1) segmentations can be customized as it does not require a pre-trained model, i.e., we train a segmentation model from scratch with manual annotations for each sequence; (2) it can be built on any semantic segmentation design such as DeepLab~\cite{deeplab}, SegNet~\cite{segnet:2015}, and Mask R-CNN~\cite{he:2017} that generates a distribution (heatmap) for each object class; (3) it can apply to general multi-camera systems (e.g., different multi-camera rigs, number of cameras, and intrinsic parameters). We validate this network by recognizing a customized semantic category per pixel from realworld visual data including non-human species and a subject of interest in social videos where attaining large-scale annotation data is infeasible. Also it quantitatively outperforms the the existing models without cross-view supervision and the model trained with annotations and shape prior in terms of accuracy and precision.

\section{Related Work}
This work lies in the intersection of semantic segmentation and multiview self-supervision, which enables learning from a small set of the labeled data possible. We briefly review these two area of study. 

\noindent\textbf{Semantic Segmentation} Semantic segmentation has been notorious for its computational complexity~\cite{fulkerson:2009} caused by spatial interactions between pixels. A seminal work by Long et al.~\cite{long:2015} has shown that such complex relationship can be effectively learned by a CNN (i.e., fully convolutional network) that encodes high level visual semantics. Albeit impressive, due to the limited network capacity and low resolution, the segmentation results misses object boundary details. Many subsequent studies have integrated a conditional random field or Markov random field~\cite{bertasius:2016,chen:2015,qi:2015} that can jointly optimize the object boundary and region. Another approach is to leverage devolutional layers similar to variational autoencoder to reconstruct full resolution segmentation~\cite{hong:2015,noh:2015}. Such advancement produces a variety of applications such as graphics~\cite{oh:2017,ren:2017,li:2018,Aksoy:2018}, autonomous driving~\cite{kundu:2014,kundu:2016}, and first-person vision~\cite{Fathi:2011}. When multiple images are used, co-segmentation is possible, i.e., segmenting common objects. Most approaches often leverage individual segmentation to correlate their visual features~\cite{kim:2012, Mukherjee:2018,Li:2018a}. Notably, multiview co-segmentation has been studied by using multiview stereo~\cite{kowdle:2012,Djelouah:2013}. Unlike these methods, our approach train a semantic segmentation networks using multiview geometry without reconstructing 3D objects, which is not sensitive to the stereo matching error.

\noindent\textbf{Multiview Self-supervision}
Learning a view invariant representation is a long-standing goal in visual recognition research, which requires to predict underlying 3D structure from a single view image. Geometrically, it is an ill-posed problem while two data driven approaches have made promising progress. (1) Direct 3D-2D supervision: for a few representative objects such as furniture~\cite{lim:2014}, vehicles~\cite{xiang_wacv14}, and human body~\cite{SMPL:2015}, their 3D models (e.g., CAD, point cloud, and mesh) exist where the 3D-2D relationship can be directly regressed. The 3D models can produce a large image dataset by projecting onto all possible virtual viewpoints where the object's pose and shape can be learned from 3D-2D pairs. This 3D model projection can be generalized to scenes measured by RGBD data~\cite{Geiger2012CVPR,Vijayanarasimhan:2017,Byravan:2016,ladisky:2014,su:2017} and graphically generated photo-realistic scenes~\cite{chen:2017,richter:2016} where visual semantics associated with 3D shape can be encoded. (2) Indirect supervision via non-rigid graph matching: to some extent, it is possible to infer the common shape and appearance from a set of single view image instances without 3D supervision. For instance, tables have a common shape expressed by four legs and planar top. Such holistic spatial relationship can be unveiled by casting it as a graph matching problem where local shape rigidity and appearance models can describe the relationship between nodes and edges~\cite{cho2013,zhou:2013,berg:2005,caetano:2009,leordeanu:2007,torresani:2008}. Further, leveraging a underlying geometric constraint between instances (e.g., cyclic consistency~\cite{Zhou:2015,zhou2016learning}, volumetric projection~\cite{drcTulsiani17,choy20163d,factored3dTulsiani17}, and kinematic chain~\cite{yan:2006,Shaji:2010,torresani:2008}) can extend the validity of graph matching. These existing approaches require many correspondences between domains that are established by manual annotations. In contrast, our approach will leverage self-supervision via multiview geometry to adapt to a novel scene with minimal manual efforts.

\begin{figure}[t]
\begin{center}
   \includegraphics[width=0.8\linewidth,height=6cm]{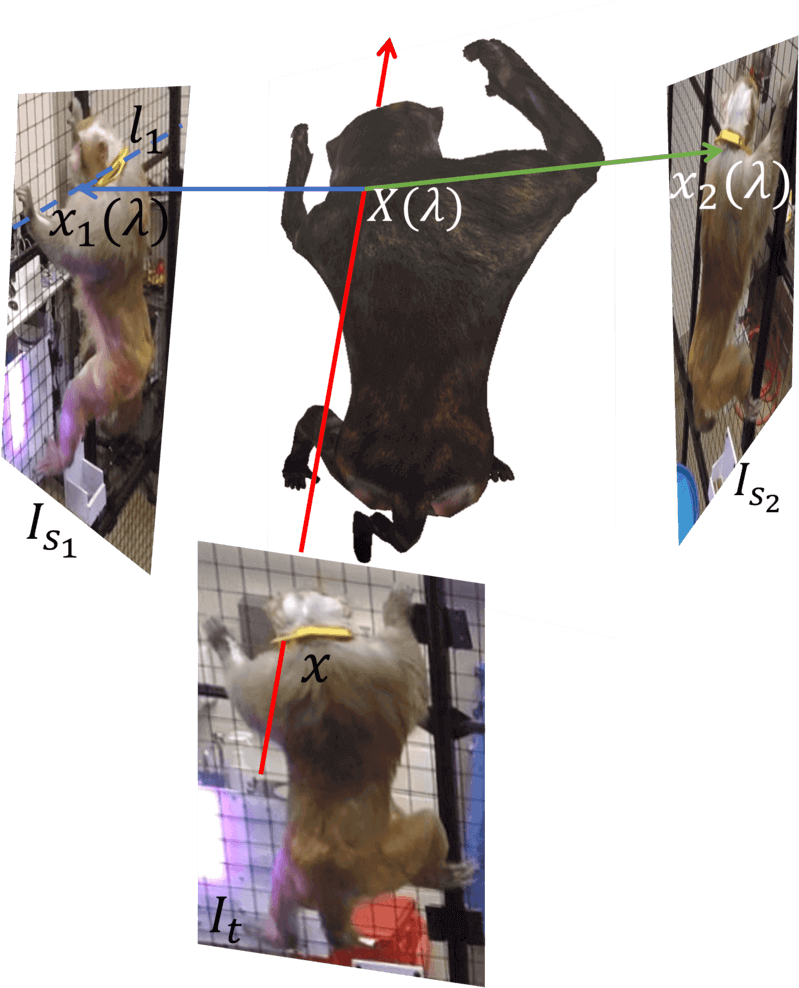}
\end{center}
   \caption{The 3D points and the corresponding 2D points on their epipolar lines can be parameterized with $\lambda$.}
\label{fig:figure}
\end{figure}
%------------------------------------------------------------------------
\section{Multiview Cross-view Supervision}

We present a semi-supervised learning framework for training a semantic segmentation model by leveraging multiview images streams where $\eta = \frac{|\mathcal{D}_L|}{|\mathcal{D}_L|+|\mathcal{D}_U|} \ll 1$ where $\mathcal{D}_L$ and $\mathcal{D}_U$ are the labeled and unlabeled data, respectively. We formulate a novel theory of rasterized multiview geometry that enforces the geometric consistency by minimizing the reprojection error of a 3D visual hull, resulting in a differentiable loss function to train a neural network. Note that we will focus on binary segmentation for a proof of concept while the multiview theory can be applied to multi-way segmentation. Also the the framework is agnostic to the design of segmentation networks where state-of-the-art models~\cite{deeplab,long:2015} can be used with a trivial modification similar to MONET~\cite{jafarian:2018}.

% enforces geometric consistency by minimizing the the segmentation area constraint violations (Section~\ref{Sec:visualhull}), allowing a cross-view supervision. 

Consider a network model that takes an input image $\mathcal{I}$ and outputs the class probabilities for each pixel, i.e., $\phi(\mathcal{I};\mathbf{w}) \in [0,1]^{W \times H \times C}$ where $W$ and $H$ are the width and height of the output distribution, respectively, and $C$ is the number of object classes. We consider binary segmentation, i.e., $|C|=2$. 

The network is parametrized by the weight $\mathbf{w}$ learned by minimizing the following loss:
\begin{align}
    \underset{\mathbf{w}}{\operatorname{minimize}}~~\mathcal{L}_L + \lambda_s \mathcal{L}_S + \lambda_p \mathcal{L}_P, \label{Eq:final}
\end{align}
where $\mathcal{L}_L$, $\mathcal{L}_S$, and $\mathcal{L}_P$ are the losses for labeled supervision, multiview cross-view supervision, and bootstrapping prior, and $\lambda_s$ and $\lambda_p$ are the weights that control their importance.  

For the labeled data, we use the sum of pixelwise cross entropy to measure the segmentation loss:
\begin{align}
    \mathcal{L}_L = -\sum_{j\in \mathcal{D}_L} \sum_{\mathbf{x}\in X} y_j(\mathbf{x})\log \left.\phi(\mathcal{I}_j)\right|_\mathbf{x},
\end{align}
where $y_j(\mathbf{x})\in \{0,1\}$ is the ground truth semantic label of the $j^{\rm th}$ labeled data at pixel location $\mathbf{x}$, and $X$ is the domain of $\mathbf{x}$. 

%-------------------------------------------------------------------------
\subsection{Shape Transfer} \label{Sec:visualhull}

Inspired by the image based shape from silhouette~\cite{matusik:2000}, we study the segmentation transfer through a 3D shape. Consider a point $\mathbf{x}\in \mathds{R}^2$ in the target image $\mathcal{I}_t$. Without loss of generality, the camera projection matrix of the target image is set to $\mathbf{P} = \mathbf{K} \left[\begin{array}{cc}\mathbf{I}_3 & \mathbf{0}\end{array}\right]$ where $\mathbf{K}$ is the intrinsic parameter. The point in an image is equivalent to a 3D ray $\mathbf{L}_\mathbf{x}\propto\mathbf{K}^{-1}\widetilde{\mathbf{x}}$ emitted from the target camera. A 3D point along the ray can be represented as $\mathbf{X}(\lambda) = \lambda\mathbf{L}_\mathbf{x}$ where any scalar depth $\lambda>0$. 

A series of projections of $\mathbf{X}(\lambda)$ onto a source image, $\mathcal{I}_{s_1}$ form the epipolar line $\mathbf{l}_1=\mathbf{F}_1\widetilde{\mathbf{x}}$ where $\mathbf{F}_1$ is the fundamental matrix between the target and source image. This indicates the point on the epipolar line can be parametrized by $\lambda$ as shown in \Cref{fig:figure}, i.e., $\mathbf{x}_1(\lambda) \in \mathbf{l}_1$\footnote{We use an abuse of notation: $\mathbf{x} \in \mathbf{l}$ is equivalent to $\widetilde{\mathbf{x}}^\mathsf{T}\mathbf{l}=0$, i.e., the point $\mathbf{x}$ belongs to the line $\mathbf{l}$}. Likewise a point $\mathbf{x}_i$ in the $i^{\rm th}$ source image $\mathcal{I}_i$ can be described accordingly.

The image based shape-from-silhouette computes a binary map $z_t:\mathds{R}^2\rightarrow \{0,1\}$ that determines the pixel being foreground if one, and zero otherwise. This binary map can be approximated by the logical operations between the binary maps from the $n$ source images ($z_{s_1},\cdots,z_{s_n}$):
\begin{align}
    \hat{z}_t(\mathbf{x}) = \left\{\begin{array}{cl}1 & \mathrm{if~} \exists~ \lambda>0  ~\mathrm{s.t.}~ \bigwedge_i z_{s_i}(\mathbf{x}_i(\lambda))=1 \\
    0 & \mathrm{otherwise}.\end{array}\right. \label{Eq:sfs}
\end{align}
The geometric interpretation of \Cref{Eq:sfs} is that the foreground map for $\mathbf{x}$ is computed by sweeping across all 3D points along the ray $\mathbf{L}_\mathbf{x}$ to see if the ray intersects with the 3D volumetric shape defined by the foreground maps from $n$ views. A key property of this foreground approximation $\hat{z}_t(\mathbf{x})$ from $n$ views that it is always inclusive of the true $z_t(\mathbf{x})$, i.e., $\{\mathbf{x}|z_t(\mathbf{x})=1\} \subseteq \{\mathbf{x}|\hat{z}_t(\mathbf{x})=1\}$.

The implication of the approximation of \Cref{Eq:sfs} is significant for the semi-supervised learning that includes the unlabeled data because it is possible to transfer the recognition belief between views through the underlying 3D shape where the label for the unlabeled data can be approximated. Inspired by this insight, we formulate a rasterized version of \Cref{Eq:sfs} to train a semantic segmentation network.

Let $P_i:\mathds{R}^2\rightarrow [0,1]$ be the foreground probability distribution of the $i^{\rm th}$ source image, i.e., $P_i(\mathbf{x}) = \left.\phi(\mathcal{I}_i;\mathbf{w})\right|_{\mathbf{x}}$. Using the probability distribution, it is possible to compute the probability over the ray $\mathbf{L}_\mathbf{x}$ by projecting the ray onto the $i^{\rm th}$ image: 
\begin{align}
    \xi_i(\lambda;\mathbf{L}_\mathbf{x}) = P_i(\mathbf{x}_i(\lambda))~~\mathrm{where}~~\mathbf{x}_i(\lambda) \in \mathbf{F}_i \widetilde{\mathbf{x}}, \label{Eq:ray_prob}
\end{align}
where $\xi_i(\lambda;\mathbf{L}_\mathbf{x})$ is the probability over the ray parametrized by the depth $\lambda$.

From \Cref{Eq:ray_prob}, the probability of a target image $P_t:\mathds{R}^2\rightarrow [0,1]$ can be approximated by a 3D line max-pooling over joint probability over $n$ views:
\begin{align}
    \hat{P}_t(\mathbf{x}) = \underset{\lambda > 0}{\operatorname{sup}}~\prod_{i=1}^n \xi_i(\lambda;\mathbf{L}_\mathbf{x}), \label{Eq:maxpool}
\end{align}
where $\hat{P}_t(\mathbf{x})$ is the foreground probability transferred from $n$ views. \Cref{Eq:maxpool} is equivalent to \Cref{Eq:sfs} where it takes the probability of a 3D point most likely being in the volumetric shape. 

Note that similar to $\hat{z}_t$, the $\hat{P}_t$ provides the upper bound of the $P_t$, i.e., $\{\mathbf{x}|P_t(\mathbf{x}) > \epsilon \} \subseteq \{\mathbf{x}|\hat{P}_t(\mathbf{x}) > \epsilon \}$. Therefore, direct probability matching using KL divergence~\cite{Kullback:1951} does not apply. Instead, we formulate a new loss $D_S$ using one-way relative cross-entropy as follow:
\begin{align}
    \mathcal{L}_S=D_S(P_t||\hat{P}_t) = \sum_{\mathbf{x}\in X} (1-\hat{P}_t(\mathbf{x})) P_t(\mathbf{x}), \label{Eq:entropy}
\end{align}
where $X$ is the range of the target image coordinate. $D_S(P_t||\hat{P}_t)$ strongly penalizes the set of pixels, $\{\mathbf{x}|\hat{P}_t(\mathbf{x}) < P_t(\mathbf{x}) \}$. \Cref{Fig:model} shows the visualization of the cross-supervision loss.

The main benefits of \Cref{Eq:entropy} are threefold. (1) Multiview segmentation involves two processes: 3D reconstruction of the shape with source views and 2D projection onto the target view. The requirement of 3D reconstruction introduces an additional estimation such as multiview~\cite{kowdle:2012,Djelouah:2013} or single view depth prediction~\cite{kanazawaHMR18,drcTulsiani17,zhou2017unsupervised} where the accuracy of the segmentation is bounded by the reconstruction quality. \Cref{Eq:entropy} integrates the 3D reconstruction and projection through the joint probability over the epipolar lines and supremum operation, which bypass the 3D reconstruction. (2) By minimizing \Cref{Eq:entropy}, it can provide a {\em pseudo-}label for the unlabeled data transferred from the labeled data. As the number of the labeled data increases, the transferred segmentation label approaches to the true label of the unlabeled data~\cite{matusik:2000,Kutulakos:2000}, which allows multiview self-superivsion, i.e., the semantic segmentation of labeled data can supervise the that of the unlabeled data. (3) Not only for the unlabeled data, but also it can correct the geometrically inconsistent segmentation label for the labeled data. This is a significant departure from the existing semantic segmentation that cannot recover erroneous segmentation label, which often arises from per-pixel manual annotations.

\begin{figure}
\begin{center}
\includegraphics[width=1\columnwidth]{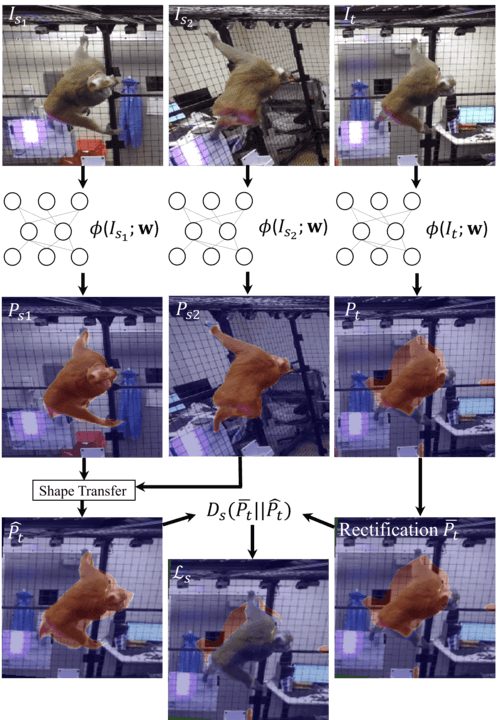}
\end{center}
\vspace{-3mm}
   \caption{We design a triplet network that allows multiview cross-supervision using shape transfer. Stereo rectification is used to simplify the max-pooling operation along the epipolar line, which reduces computational complexity and sampling aliasing.}
\label{Fig:model}
% \vspace{-2mm}
\end{figure}

\subsection{Degenerate Case Analysis} \label{Sec:degenerate}
\Cref{Eq:entropy} has a degenerate case: a trivial solution $P_t=0$ is the global minimizer. Therefore, when the unlabeled data sample is used for the target view, the cross-view supervision via shape transfer based on the labeled data is not possible, i.e., $\hat{P}_U = P_U^+ > P_U$.

\begin{thm} There exists the lower bound of the probability of the unlabeled data sample, $P_U^-$.  \label{thm:thm1}
\end{thm}
\begin{proof}
Consider an inverse shape transfer for the unlabeled data in \Cref{Eq:maxpool}, $\phi_U(\lambda;\mathbf{L}_\mathbf{x})$, to explain the first labeled data sample $P_{L_1}$:
\begin{align}
    P_{L_1}(\mathbf{x}) = \underset{\lambda > 0}{\operatorname{sup}}~\xi_U(\lambda;\mathbf{L}_\mathbf{x})\prod_{i=2}^n \xi_{L_i}(\lambda;\mathbf{L}_\mathbf{x}), \label{Eq:proof1}
\end{align}
where $P_{L_i}$ is the probability of the $i^{\rm th}$ labeled data. Since the supremum in Equation~(\ref{Eq:proof1}) is a non-decreasing function with respect to $\xi_U(\lambda;\mathbf{L}_\mathbf{x})$, there exists $\xi_U^-(\lambda;\mathbf{L}_\mathbf{x}) < \xi_U(\lambda;\mathbf{L}_\mathbf{x})$ that cannot explain $P_{L_1}(\mathbf{x})$:
\begin{align}
    P_{L_1}(\mathbf{x}) > \underset{\lambda > 0}{\operatorname{sup}}~\xi_U^-(\lambda;\mathbf{L}_\mathbf{x})\prod_{i=2}^n \xi_{L_i}(\lambda;\mathbf{L}_\mathbf{x}). 
\end{align}
Therefore, there exists the lower bound of $P_U$.
\end{proof}
From Theorem~\ref{thm:thm1}, \Cref{Eq:entropy} can provide both upper and lower bounds of the unlabeled data if used as the target and source views, i.e., $P_U^-<P_U \leq P_U^+$, and $P_U^-$ asymptotically approaches to $P_U^+$ as the number of labeled views increases~\cite{matusik:2000,Kutulakos:2000}, i.e., $\lim_{|\mathcal{D}_L|\rightarrow \infty} (P_U^+ - P_U^-) = 0$.  \Cref{Fig:overlap} shows the upper bound becomes tighter as the number of labeled data increases

\begin{figure}
  \centering  
    \includegraphics[width=0.47\textwidth]{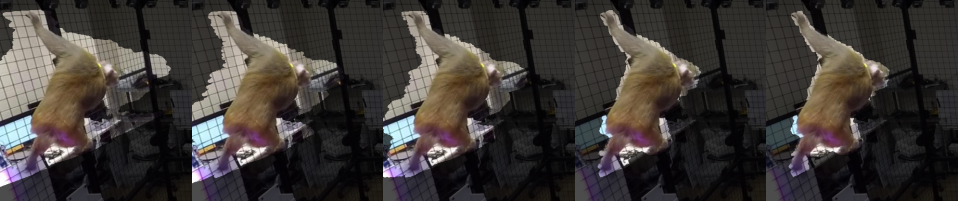}
  \caption{The upper bound of the probability for the unlabeled data becomes tighter as the number of the labeled data increases.} 
  \label{Fig:overlap}
\end{figure}

We leverage this asymptotic convergence of the shape transfer to self-supervise the unlabeled data, i.e., the unlabeled data are fed into both the target and source views, which allows the gradient induced by the error in the loss function of \Cref{Eq:entropy} can be backpropagated through the neural network to reduce the gap between $P_U^+$ and $P_U^-$. 

\subsection{Cross-view Supervision via Shape Transfer}

In practice, embedding \Cref{Eq:entropy} into an end-to-end neural network is not trivial because (a) a new max-pooling operation over oblique epipolar lines needs to be defined; (b) sampling interval for max-pooling along the line is arbitrary, i.e., uniform sampling does not encode geometric meaning such as depth; and (c) sampling interval across different epipolar line parameters is also arbitrary, which may introduce sampling aliasing. This leads to irregular segmentation probability distribution transfer based on the fundamental matrix. 

We introduce a new operation inspired by stereo rectification, which warps the segmentation probability distribution such that the epipole is transformed to a point at infinity, i.e., the epipolar lines become parallel (horizontal). This rectification allows converting the oblique line max-pooling into regular row-wise max-pooling. 

\Cref{Eq:ray_prob} can be re-written by rectifying the probability distribution of the source view with respect to the target view:
\begin{align}
    \overline{\xi}_1(u;\mathbf{L}_\mathbf{x}) = \overline{P}_1\left(\left[\begin{array}{c}u\\v_1\end{array}\right]\right)~\mathrm{where}~\mathbf{K}\mathbf{R}_1\mathbf{K}^{-1}\widetilde{\mathbf{x}} \propto \left[\begin{array}{c}x\\v_1\\1\end{array}\right], \nonumber
\end{align}
where $\mathbf{K}\mathbf{R}_1\mathbf{K}^{-1}\widetilde{\mathbf{x}}$ is the rectified coordinate of the target view, $\mathbf{R}_1\in SO(3)$ is the relative rotation for the rectification. See Appendix for more details. Note that $\xi$ is no longer a function of the depth scale $\lambda$ but the $x$ coordinate (disparity), which eliminates irregular sampling across pixels with the $y$ coordinate $v_1$. 

The key advantage of this rectification is that the $x$ coordinate of the $i^{\rm th}$ view can still be parametrized by the same $u$, i.e., the coordinate is linearly transformed to from the first view to the rest views:
\begin{align}
\overline{\xi}_i(a_i u+b_i;\mathbf{L}_\mathbf{x}) = \overline{P}_i\left(\left[\begin{array}{c}a_i u+b_i\\v_i\end{array}\right]\right) \nonumber
\end{align}
where $a_i$ and $b_i$ are the linear re-scaling factor and bias between the first and $i^{\rm th}$ views accounting for camera intrinsics and cropping parameters. $\overline{\phi}_i$ is computed by the rectified probability of the $i^{\rm th}$ view $\overline{P}_i$ with respect to the target view. See Appendix for more details. This simplifies the supremum operation over the 3D ray in \Cref{Eq:maxpool} to the max operation over the $x$ coordinates:
\begin{align}
    \hat{P}_t(\mathbf{x}) = \underset{u\in [0,W]}{\operatorname{max}}~\overline{\xi}_1(u;\mathbf{L}_\mathbf{x})\prod_{i=2}^n \overline{\xi}_i(a_i u+b_i;\mathbf{L}_\mathbf{x}). \label{Eq:maxpool1}
\end{align}

Our semi-supervised learning framework has Siamese network configuration which consists of four same segmentation models with shared weights. The first network is fully supervised which only learns the labeled data from their annotations. The triplet networks take three different images in the same frame, and they can be either labeled or unlabeled images. \Cref{Fig:model} illustrates the triplet network that minimizes the cross-view supervision loss by applying stereo rectification and shape transfer where the first two images serve as source views and third image is target view. The foreground probability distributions of the first source image and target view are rectified to reduce the sampling aliasing and computational complexity for computing the fundamental matrices. The foreground probability of each point in the target view transferred from two source views is calculated using \Cref{Eq:maxpool1}. The cross-view supervision loss computed by \Cref{Eq:entropy} is propagated back to four networks. In the actual implementation, each image is used as source view to supervise other two images and is supervised by one image pair simultaneously. The degenerate case for unlabeled data discussed in \Cref{Sec:degenerate} can be avoided once it is input to this triplet network. \Cref{Fig:vs} shows that one target view can be supervised by multiple different image pairs during the actual training.
% ．The tightness of the upper bound is determined by their relative camera poses. 

\begin{figure}
\begin{center}
\includegraphics[width=1\columnwidth]{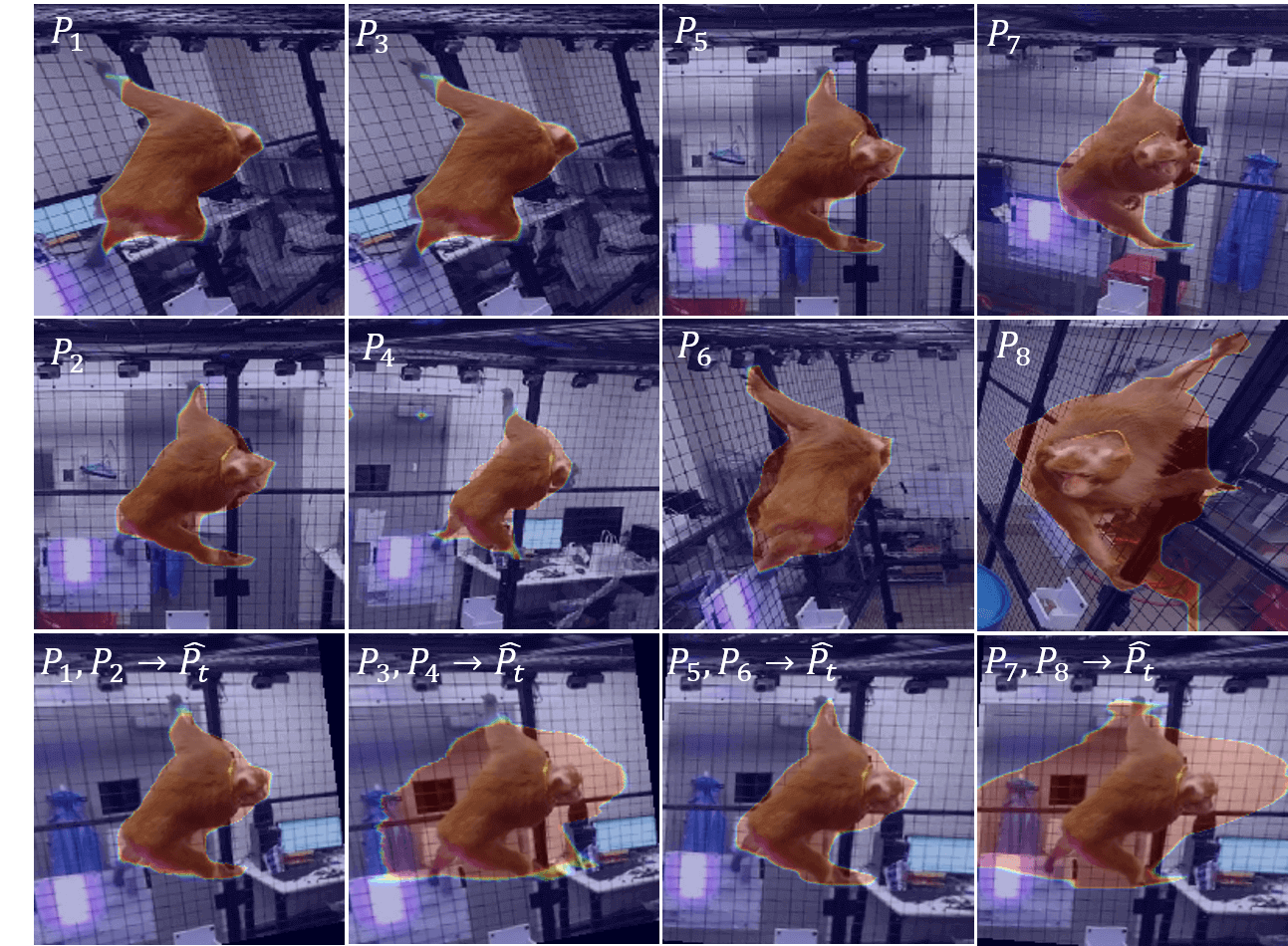}
\end{center}
\vspace{-3mm}
   \caption{Different image pairs (top two rows) can be used to supervise one target view (bottom). We use such multiple triplets to supervise each other's view.}
\label{Fig:vs}
% \vspace{-2mm}
\end{figure}

\begin{figure*}[t]
\centering 
\includegraphics[width=1.7\columnwidth]{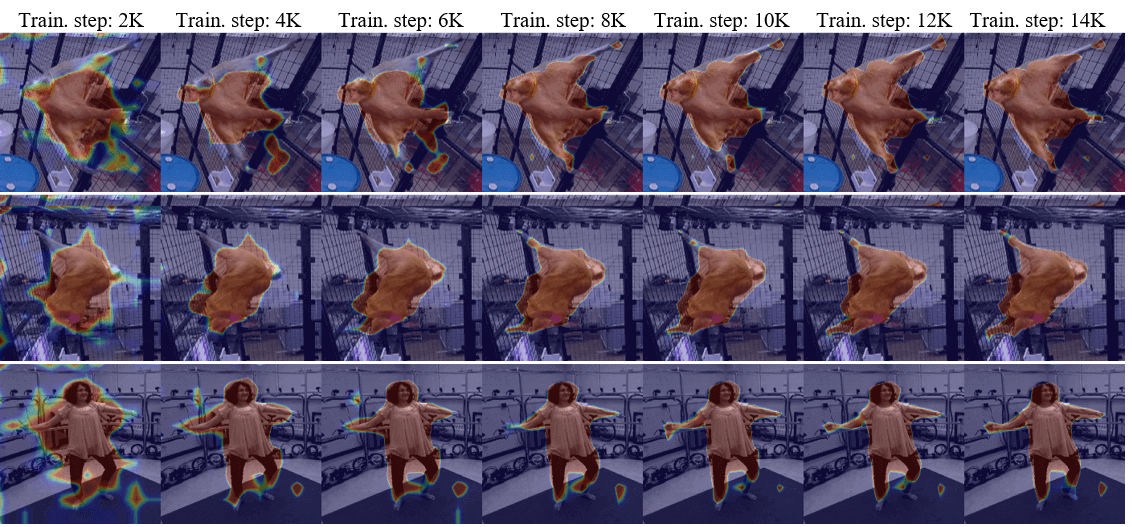}
\vspace{-2mm}
   \caption{We visualize the prediction result of our semi-supervised framework on unlabeled data every 2000 training iterations.}
\label{Fig:progress}
\end{figure*}

\subsection{Bootstrapping Prior} \label{sec:prior}
\Cref{Eq:sfs} is often highly effective to generate a prior for 3D shape given the binary label. Inspired by multiview bootstrapping~\cite{simon:2017,jafarian:2018}, we approximate the 3D shape using the pre-trained neural network $\phi$. Note that unlike keypoint detection, RANSAC~\cite{Fischler:1981} outlier rejection approaches cannot be applied because pixel correspondences are not available for semantic segmentation. We binarize the probability of the foreground segment to compute the $i^{\rm th}$ source binary map $z_{s_i}(\mathbf{x}) = 1$ if $P_i(\mathbf{x})>0.5$, and zero otherwise. Using all source binary maps, a pseudo-binary map for the $j^{\rm th}$ unlabeled data $\hat{z}_j$ can be computed and used for the bootstrapping prior, i.e.,
\begin{align}
    \mathcal{L}_P = \sum_{j \in \mathcal{D}_U} \sum_{\mathbf{x}\in X} (1-\hat{z}_j(\mathbf{x})) P_j(\mathbf{x}) \label{Eq:prior}
\end{align}
Similar to \Cref{Eq:entropy}, $\hat{z}_j$ provides the superset of the ground truth, which requires the one-way relative cross entropy as a prior loss.

\begin{figure*}
\centering
\includegraphics[width=1\textwidth]{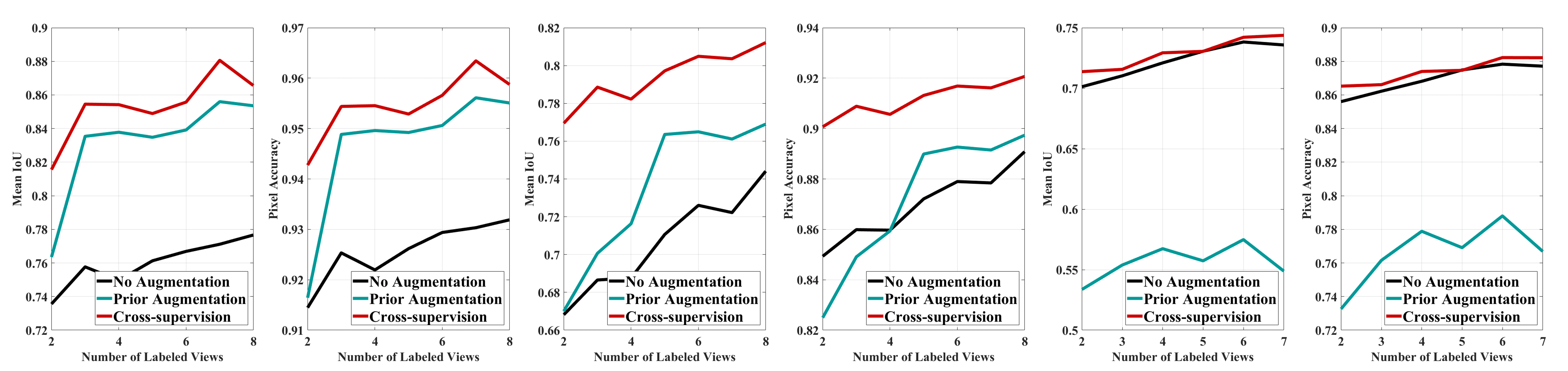}
\vspace{-4mm}
   \caption{The mean IoU and pixel accuracy of semi-supervised model tested on the unlabeled monkey, dance, and social event dataset.}
\label{Fig:acc}
\end{figure*}

\begin{table*}[h!]
\centering
\footnotesize
\resizebox{\textwidth}{!}{%
  \begin{tabular}{l|ccc|ccc|ccc|ccc|ccc|ccc}
    \toprule
    \multirow{3}{*}{Model} &
      \multicolumn{3}{c|}{Monkey (IoU)} &
      \multicolumn{3}{c|}{Dance (IoU)} &
      \multicolumn{3}{c|}{Social (IoU)} &
      \multicolumn{3}{c|}{Monkey (Pixel Acc.)} &
      \multicolumn{3}{c|}{Dance (Pixel Acc.)} &
      \multicolumn{3}{c|}{Social (Pixel Acc.)} \\
      & {2} & {5} & {8} & {2} & {5} & {8} & {2} & {5} & {7} & {2} & {5} & {8} & {2} & {5} & {8} & {2} & {5} & {7} \\
      \midrule
    No aug & 0.735 & 0.761 & 0.776 & 0.668 & 0.710 & 0.744 & 0.701 & 0.730 & 0.735 & 0.914 & 0.926 & 0.931 & 0.849 & 0.872 & 0.890 & 0.856 & 0.874 & 0.877 \\
    Prior & 0.763 & 0.834 & 0.853 & 0.670 & 0.763 & 0.769 & 0.533 & 0.557 & 0.548 & 0.916 & 0.949 & 0.955 & 0.824 & 0.889 & 0.897 & 0.732 & 0.769 & 0.766 \\
    Cross & 0.815 & 0.848 & 0.865 & 0.769 & 0.797 & 0.812 & 0.713 & 0.730 & 0.743 & 0.942 & 0.952 & 0.958 & 0.900 & 0.913 & 0.920 & 0.865 & 0.874 & 0.882\\
    \bottomrule
  \end{tabular}%
  }
  \vspace{-2mm}
  \caption{Mean IoU and pixel accuracy result on different datasets with different number of labeled views}
\label{table:iou}
\end{table*}

\section{Result}
We validate our semi-supervised semantic segmentation framework using real-world data on human and non-human species including a subject of interest in social videos with three different multi-camera systems. Monkey, dancer and social event subjects are captured by 69, 35, and 18 cameras, respectively. One monkey was crawling against the cage in the video, and the array of cameras were placed in the cage ceiling. An Indian dancer was performing solo dance captured by 69 cameras in three layers with different heights. In the social event videos, a group of dancers were performing Hip-hop dance, and they were surround by the audiences holding hand-held cameras.

% These videos are captured by 18 walking people holding a camera in their hands. 

To evaluate the flexibility, we build a model per subject without a pre-trained model. The DeepLab v3\cite{deeplab} network is used to build the fully supervised and semi-supervised triplet network. Our segmentation network takes an input image (200$\times$200), and outputs two distribution heatmaps for foreground and background (200$\times$200). In the training, we use the batch size 5 for fully supervised network and 3 for triplet network, learning rate $10^{-5}$, batch norm epsilon $10^{-5}$, and batch norm decay 0.9997. We use an ADAM optimizer of TensorFlow with nVidia GTX 1080. 

We randomly sample 16, 16, and 14 cameras from monkey, dance and social event datasets and manually annotate the 20 frames in half sampled cameras. We conduct multiple experiments using different number of labeled views for bootstrapping prior and cross-supervision from two to half number of the sampled cameras. We compare our approach with two different baseline algorithms. For all algorithms, we evaluate the performance on the unlabeled data. (1) \textbf{No augmentation}: we use the manually annotated images to train the network. (2) \textbf{Prior augmentation}: the prior of unlabeled data is generated the way discussed in \Cref{sec:prior}. We train the network using both annotations and prior.

\subsection{Quantitative  Result}
We evaluate our approach based on two metrics: mean IoU and mean pixel accuracy. \Cref{Fig:acc} shows mean IoU and mean pixel accuracy performance on monkey, dance, and social event subjects using different number of labeled view, and no pre-trained model is used. \Cref{table:iou} reports the numerical results from \Cref{Fig:acc}. All the figures display the same trend that the accuracy increases as the number of labeled views increases. 

Our cross-supervision (red) model exhibits accurate segmentation for all subjects, which outperforms 2 baselines. Cross-supervision model and the model trained with prior augmentation both perform better than the model trained with annotation only on the monkey dataset while the model trained with prior augmentation and model trained with annotation have similar performance on the dance dataset. We observe that the most labeled views selected for the monkey dataset can generate a tight upper bound for the monkey in the target views. The labeled cameras sampled for dance video have very close distances and similar angles; therefore, multiple upper bounds constructed for the dancer are very loose. However, as we observe in \Cref{Fig:acc}, if the weight on boostrapping prior is set to be lower, the cross-supervision is able to correct the prior. 

% It is unusual that the model trained with annotation outperforms the model trained with prior augmentation for the social event data. After careful check on this dataset, we found that since the cameras were moving during the capture, the extrinsic parameters of the cameras are not synchronized accurately enough, especially for the dancers who are relatively small in the images. 

\subsection{Qualitative  Result}
The qualitative comparison can be found in \Cref{Fig:compare}. This figure shows the prediction results on the unlabeled data using three models. Our cross-supervision method is able to correct the segmentation errors in the two baselines by leveraging multiview images jointly. This becomes more evident on the boundaries or protruding body parts, e.g., monkey's paws and tails, human's legs and hands. \Cref{Fig:progress} shows how the semi-supervised network progresses on the unlabeled data during the training. We can see that as the training iteration increases, the prediction becomes tighter around the subject while the protruding body parts can be predicted more accurately. 

\section{Discussion}
We present a new semi-supervised framework to train a semantic segmentation network by leveraging multi-view image streams. The key innovation is a method of shape belief transfer---using segmentation belief in one image to predict that of the other image through epipolar geometry analogous to shape-from-silhouette. The shape belief transfer provides the upper and lower bounds of the segmentation for the unlabeled data. We introduce a triplet network which is used to embed computing of transferred shape. We also use multi-view image streams to bootstrap the unlabeled data for training data augmentation. 

% We demonstrate that our framework outperforms baseline approaches in terms of mean IoU and mean pixel accuracy.  

\begin{figure*}
\centering 
\includegraphics[width=2\columnwidth]{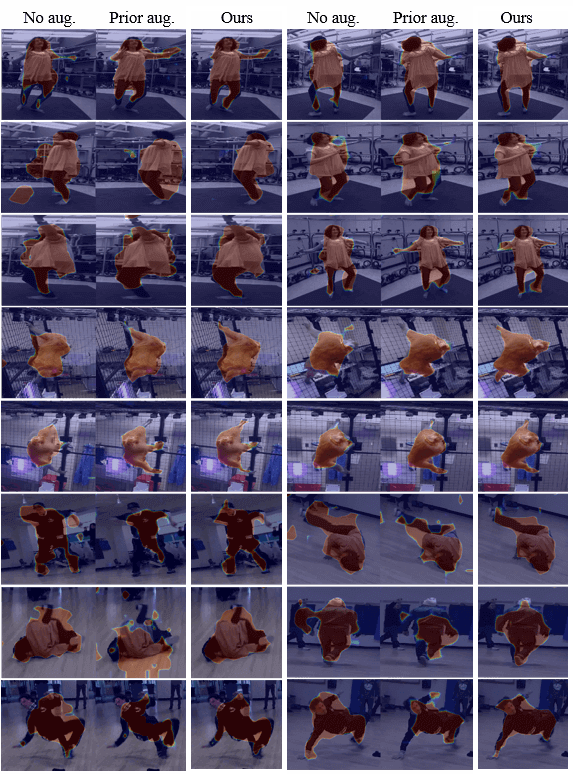}
   \caption{We qualitatively compare our semi-supervised framework with 2 baseline algorithms on dance, monkeys, and social event.}
\label{Fig:compare}
\end{figure*}

\clearpage
{\small
\bibliographystyle{ieee}
\bibliography{egpaper_for_review}
}

\newpage
\appendix 

\section{Cropped Image Correction and Stereo Rectification} \label{Sec:rect}

\begin{figure}
\centering 
\includegraphics[width=1\columnwidth]{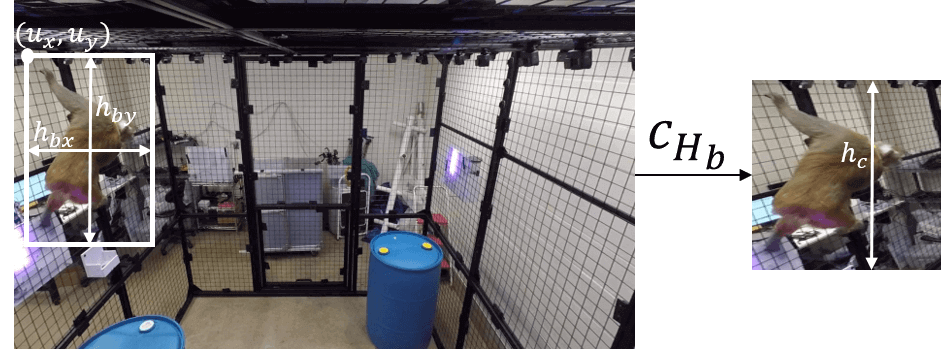}
\caption{A cropped image is an input to the network where the output is the segmentation distribution with the same size. To rectify the segmentation distribution (heatmap), a series of image transformations need to be applied}
\label{Fig:exp}
\end{figure}

We warp the segmentation distribution using stereo rectification. This requires a composite of transformations because the rectification is defined in the full original image. The transformation can be written as:
\begin{align}
    ^{\overline{h}}\mathbf{H}_h = \left(^{\overline{c}}\mathbf{H}_{\overline{b}}\right)\mathbf{H}_r \left(^c\mathbf{H}_b\right)^{-1}.
\end{align}
The sequence of transformations takes a segmentation distribution of the network output $P$ to the rectified segmentation distribution $\overline{P}$: cropped and resized image$\rightarrow$original image$\rightarrow$rectified image$\rightarrow$rectified cropped and resize image. 

Given an image $\mathcal{I}$, we crop the image based on the bounding box as shown in \Cref{Fig:exp}: the left-top corner is $(u_x, u_y)$ and the height is $h_b$. The transformation from the image to the bounding box is:
\begin{align}
    ^c\mathbf{H}_b = \begin{bmatrix}s_x &  0 & -s_x u_x\\ 0 & s_y & -s_y u_y\\0 & 0 & 1\end{bmatrix}
\end{align}
where $s_x = h_{c}/h_{bx}$ and $s_y = h_{c}/h_{by}$. It corrects the aspect ratio factor. $h_{c}=200$ is the width and height of the cropped image, which is the input to the network. The network output have the same resolution as the input. The rectified transformations $\left(^{\overline{c}}\mathbf{H}_{\overline{b}}\right)$ can be defined in a similar way. 

Given the cropping factors, we derive $v_i$ and the re-scaling factor of $a_i$ and $b_i$ in the following Equation in Section 3.3:
\begin{align}
\overline{\xi}_i(a_i u+b_i;\mathbf{L}_\mathbf{x}) = \overline{P}_i\left(\left[\begin{array}{c}a_i u+b_i\\v_i\end{array}\right]\right), \nonumber
\end{align}
where $v_i$ is the y coordinate of the rectified image that corresponds to $(u_1,v_1)$. $v_i$ can be computed by transforming $(u_1,v_1)$ to the $i^{\rm th}$ rectified coordinate: 
\begin{align}
 \left[\begin{array}{c}u_i\\v_i\\1\end{array}\right] = 
 \left(^{\overline{c_{i}}}\mathbf{H}_{\overline{b_{i}}}\right) ^{r_i}\mathbf{H}_o \left(^{r_1}\mathbf{H}_o\right)^{-1} \left(^{c_{1}}\mathbf{H}_{b_{1}}\right)^{-1} \left[\begin{array}{c}u_1\\v_1\\1\end{array}\right], \nonumber
\end{align}

where ${^{r_i}\mathbf{H}_o}$ is the homography that rectifies the original target view with respect to the $i^{\rm th}$ camera. The point in the first cropped and rectified image $(u_1, v_1)$ is transformed to $(u_i, v_i)$.

For $a_i$ and $b_i$, 
\begin{align}
 a_i &= \frac{W_1 \cos\theta}{W_i} \\
 b_i &= - \frac{W_1 u_o^1 \cos\theta}{W_i} + u_o^i,
\end{align}
where $\theta = \cos^{-1} \frac{trace(\mathbf{R}_o^{1\rightarrow i})-1}{2}$. $W_i$ is the distance (baseline) between the $i^{\rm th}$ camera and target camera. $u_o^i$ is the point in target view rectified with respect to $i^{\rm th}$ view. $\mathbf{R}_o^{1\rightarrow i}$ is the difference between the rotations of target view rectified with respect to the first view and $i^{\rm th}$ view.

\section{Qualitative  Result}\label{Sec:quali}

We validate our semi-supervised semantic segmentation framework using three real-world datasets: monkey, dancer and the subjects in social videos. In the social event videos, a group of dancers were performing Hip-hop dance, and they were surround by the audiences holding hand-held cameras. An Indian dancer was performing solo dance captured by 69 cameras in three layers with different heights. One monkey was crawling against the cage in the video, and the array of cameras were placed in the cage ceiling. 

\Crefrange{Fig:c1}{Fig:c8} shows the prediction results on the unlabeled data using three models. \Crefrange{Fig:compare1}{Fig:compare8} shows how the semi-supervised network progresses on the unlabeled data during the training. \Cref{Fig:b2} shows some failure cases of our segmentation framework. 

One possible reason of the failures on social data is that since the people who hand-held the cameras were walking around when they captured the videos, the synchronization is not accurate enough; therefore the camera rotation and location data is very noisy, which causes the shape belief transfer incorrect. Other reasons for failures can be that there are no enough source image pairs which are able to construct tight upper bound for subjects. Or the weight on prior is too small to affect the predictions. 

\begin{figure*}
\centering 
\includegraphics[width=2\columnwidth]{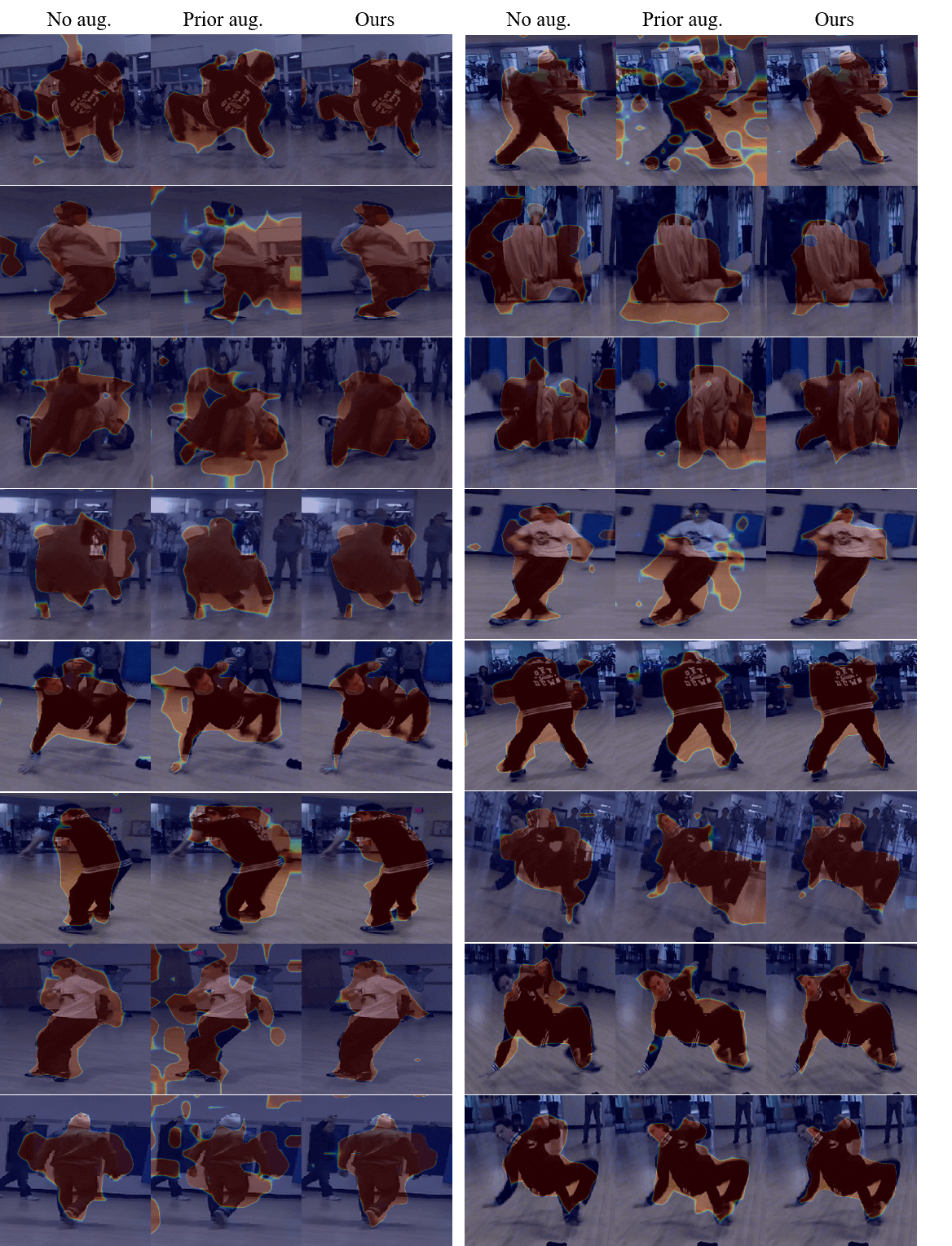}
\caption{Qualitative result of multiview segmentation.}
\label{Fig:c1}
\end{figure*}

\begin{figure*}
\centering 
\includegraphics[width=2\columnwidth]{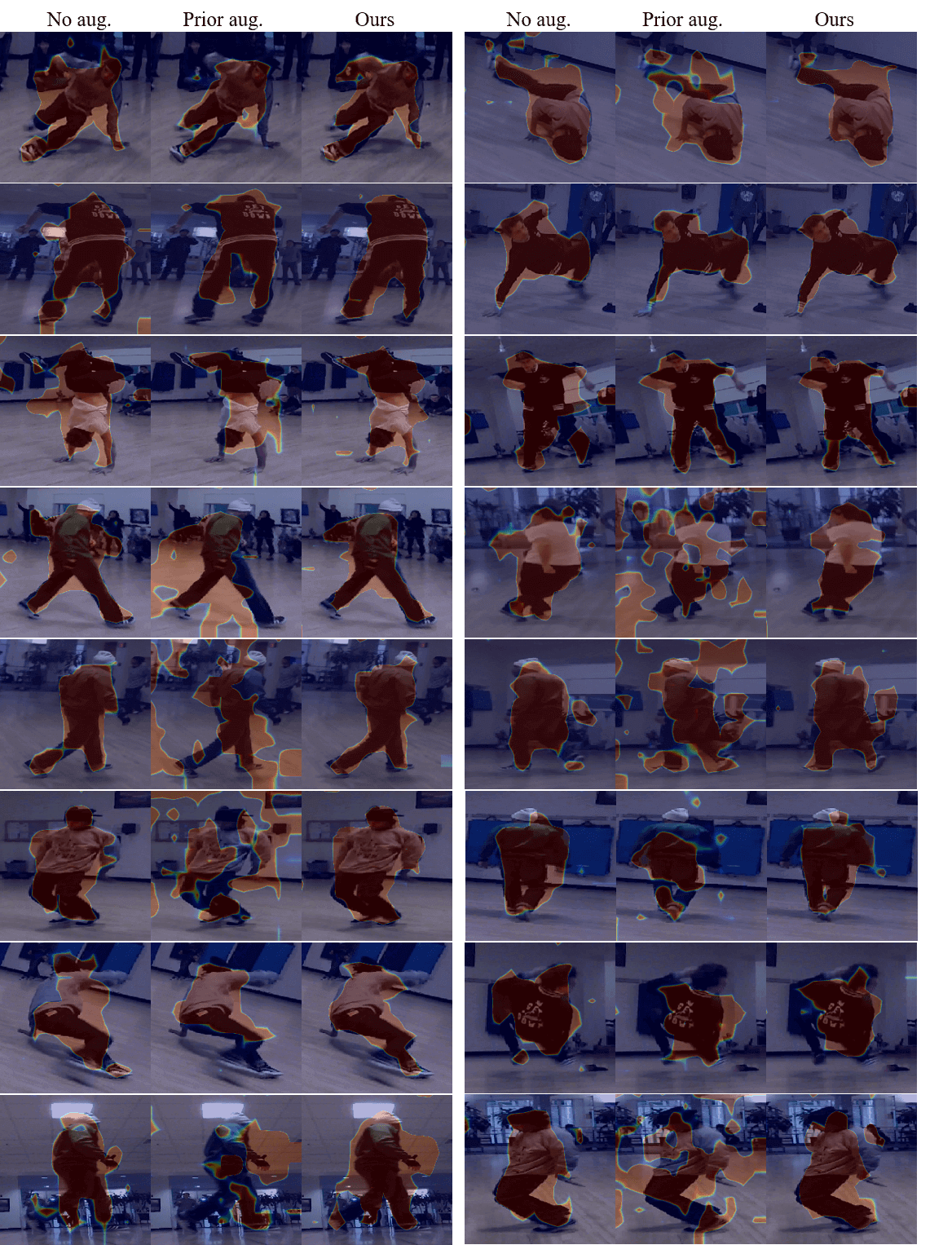}
\caption{Qualitative result of multiview segmentation.}
\label{Fig:c2}
\end{figure*}

\begin{figure*}
\centering 
\includegraphics[width=2\columnwidth]{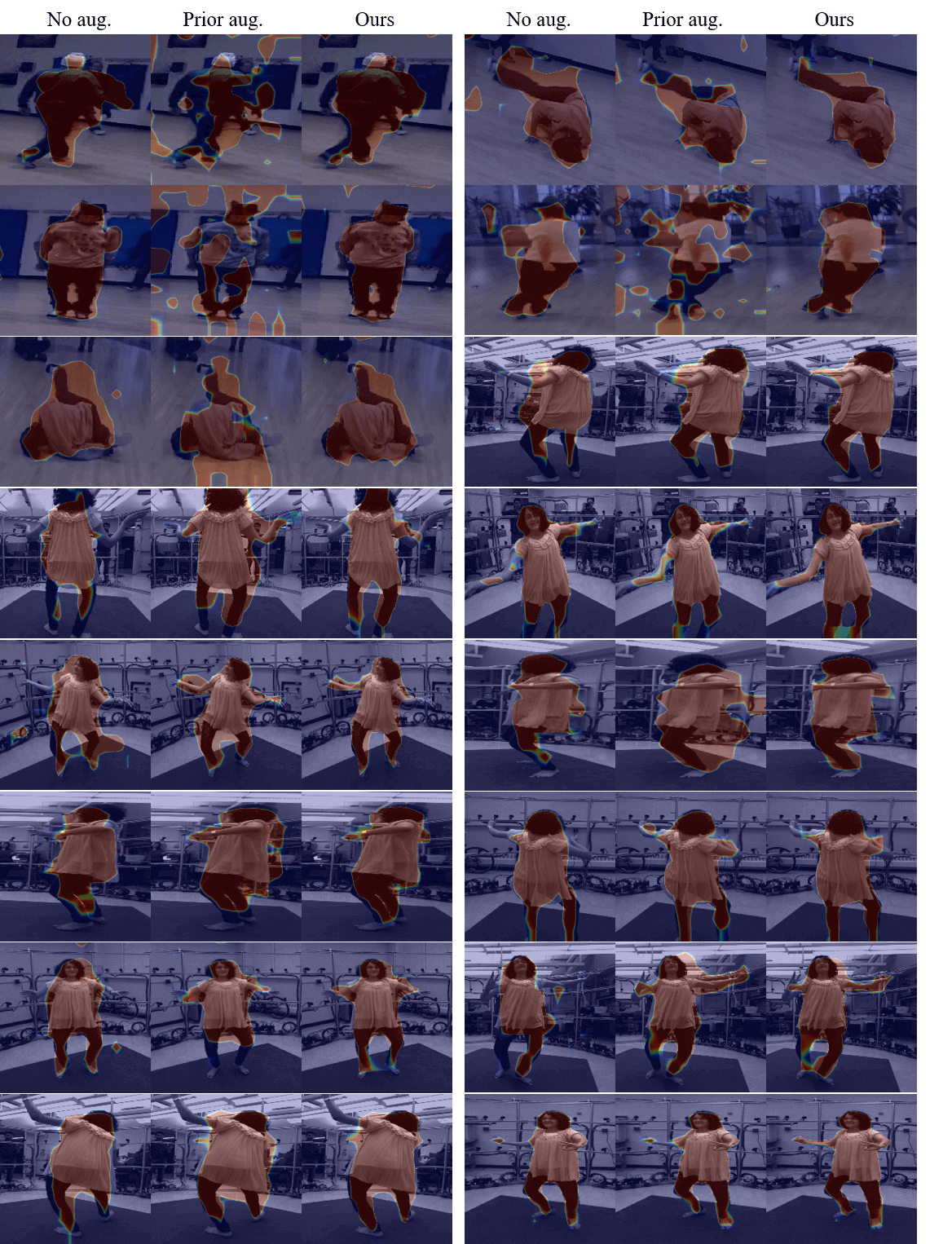}
\caption{Qualitative result of multiview segmentation.}
\label{Fig:c3}
\end{figure*}

\begin{figure*}
\centering 
\includegraphics[width=2\columnwidth]{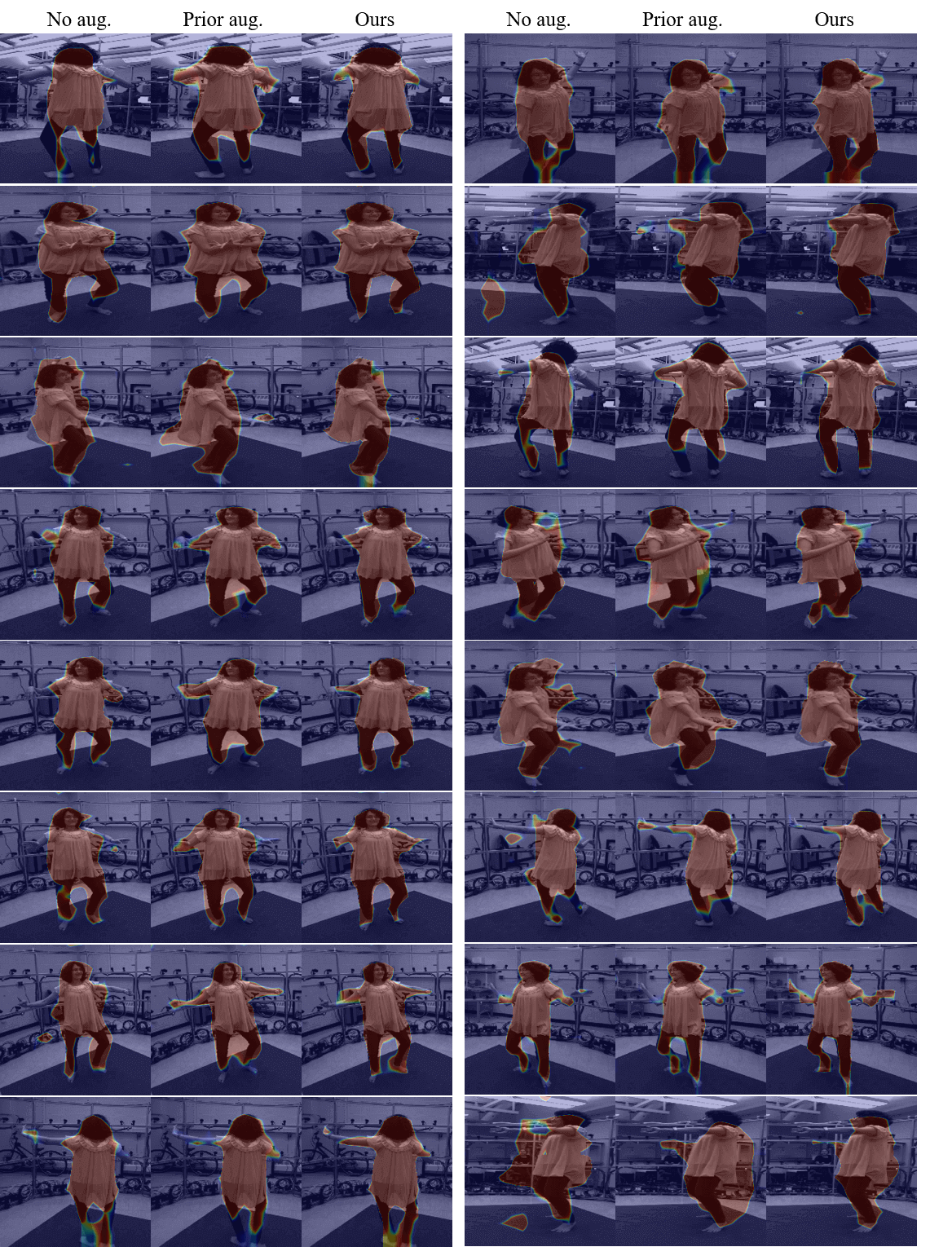}
\caption{Qualitative result of multiview segmentation.}
\label{Fig:c4}
\end{figure*}

\begin{figure*}
\centering 
\includegraphics[width=2\columnwidth]{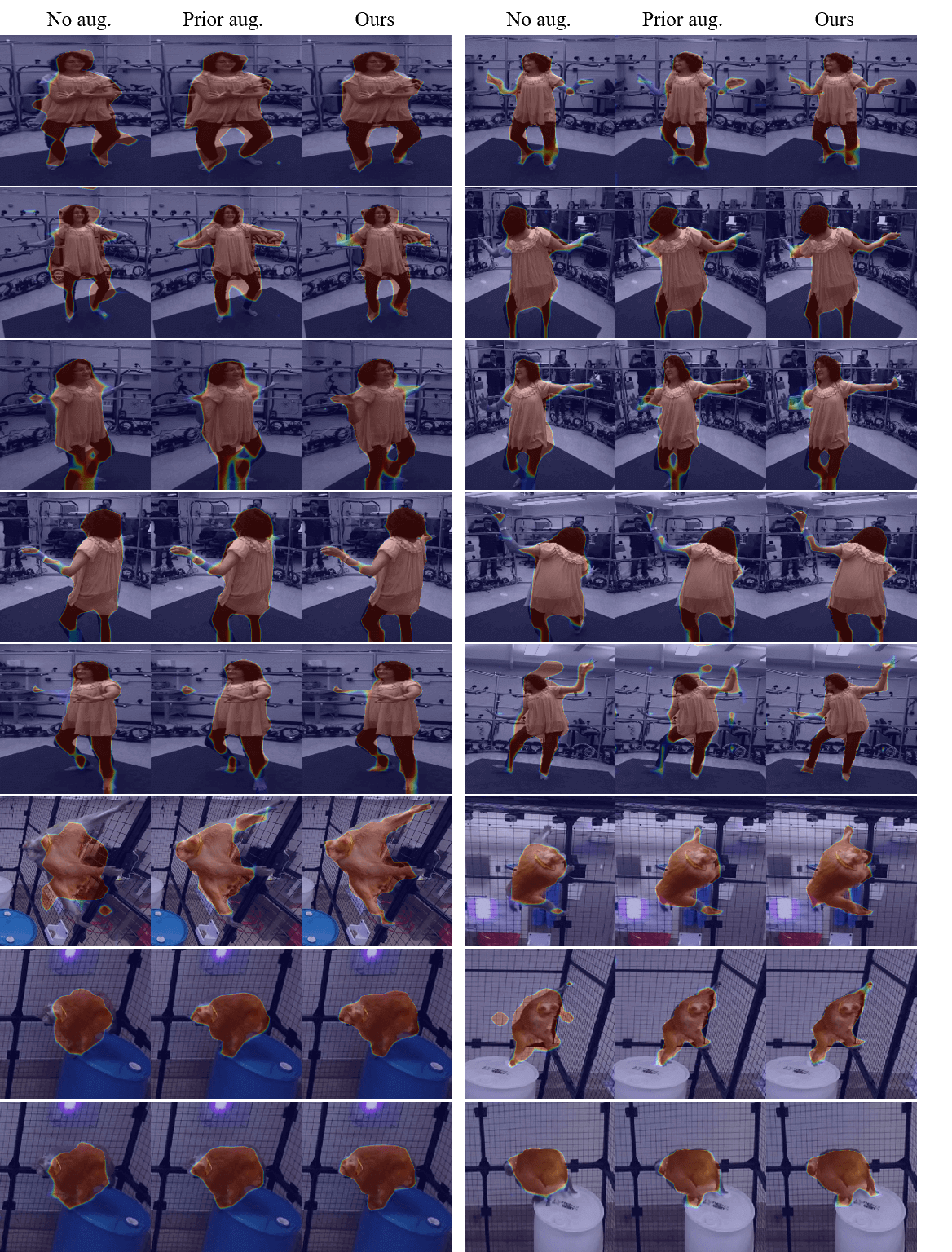}
\caption{Qualitative result of multiview segmentation.}
\label{Fig:c5}
\end{figure*}

\begin{figure*}
\centering 
\includegraphics[width=2\columnwidth]{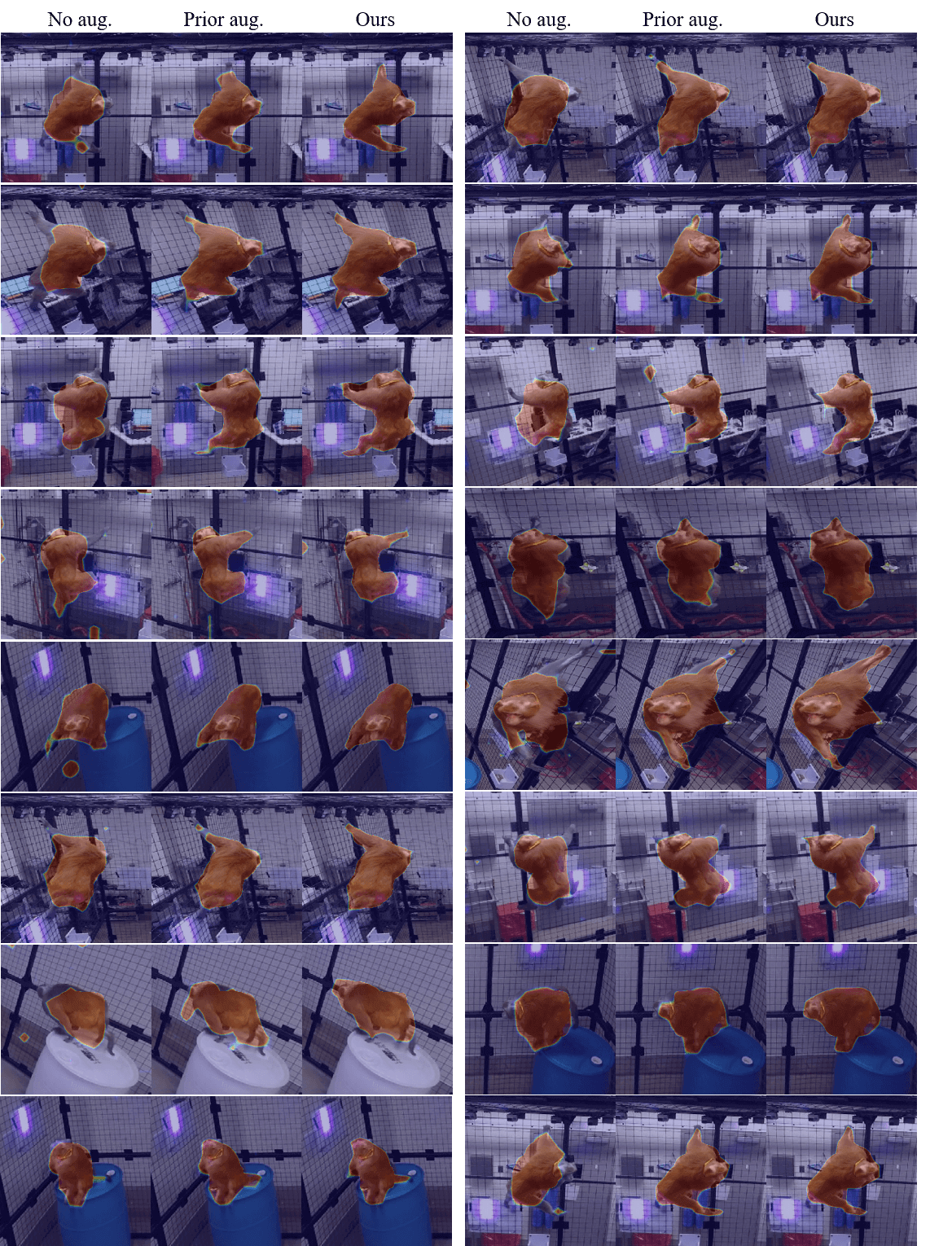}
\caption{Qualitative result of multiview segmentation.}
\label{Fig:c6}
\end{figure*}

\begin{figure*}
\centering 
\includegraphics[width=2\columnwidth]{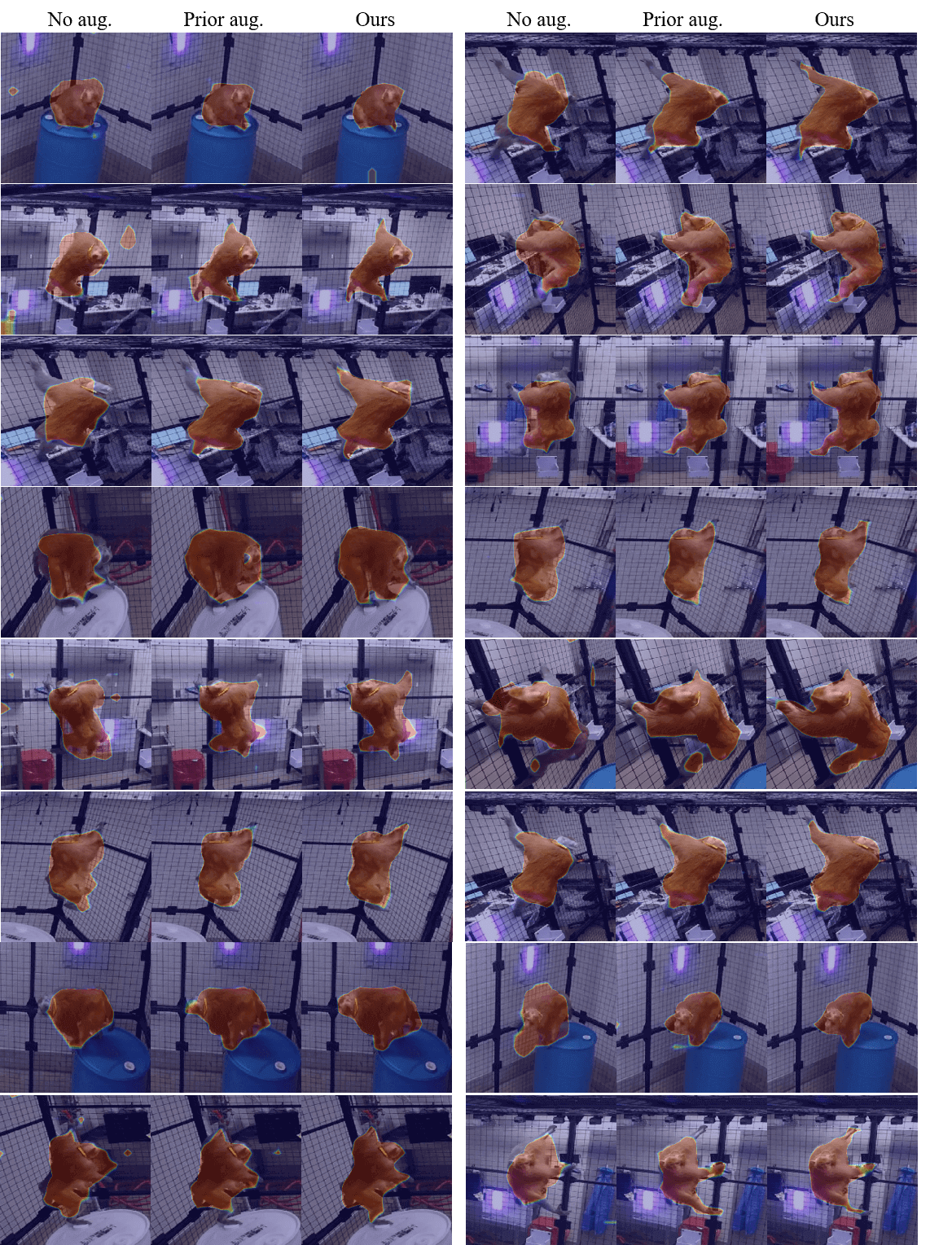}
\caption{Qualitative result of multiview segmentation.}
\label{Fig:c7}
\end{figure*}

\begin{figure*}
\centering 
\includegraphics[width=2\columnwidth]{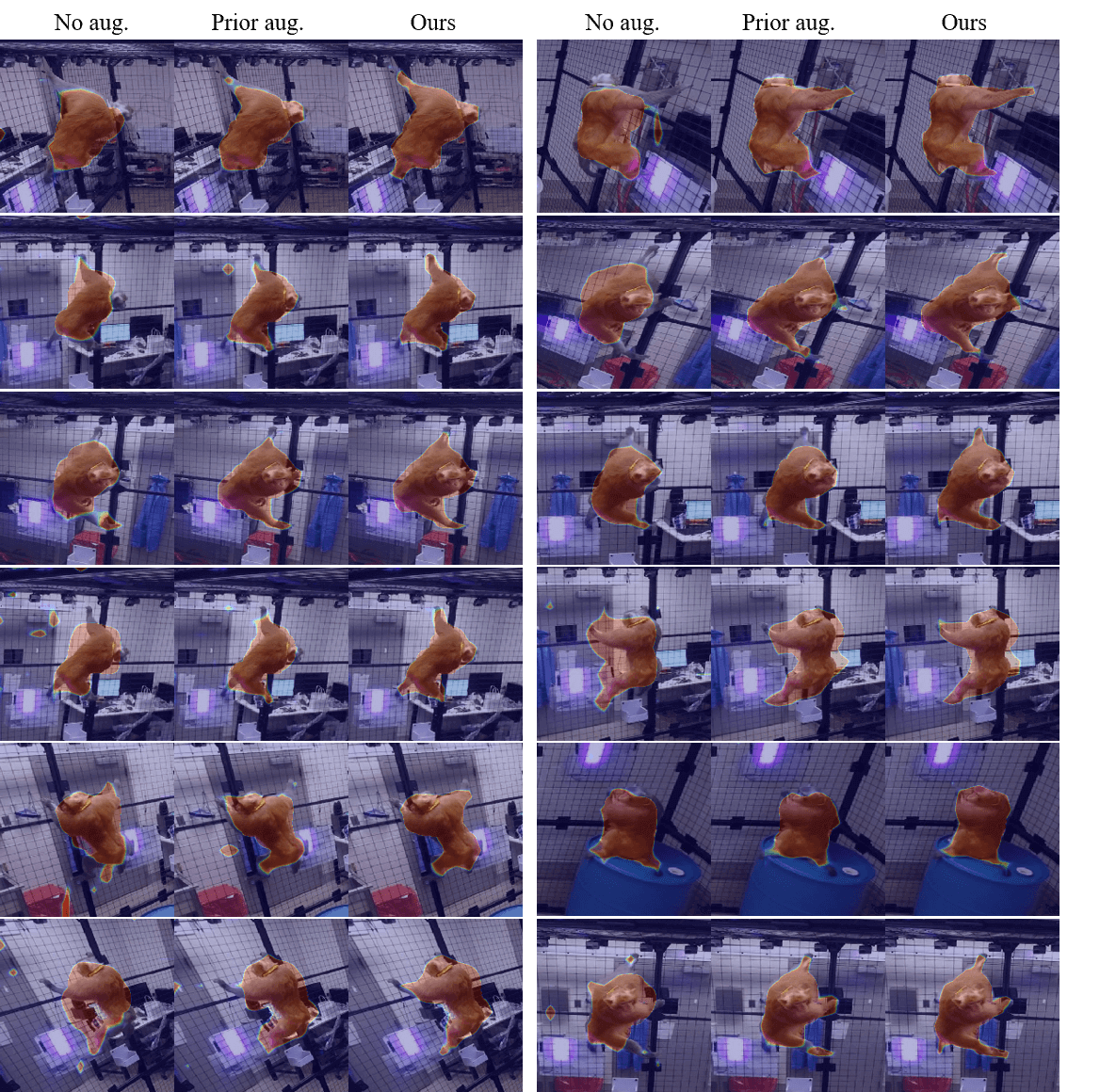}
\caption{Qualitative result of multiview segmentation.}
\label{Fig:c8}
\end{figure*}

\begin{figure*}
\centering 
\includegraphics[width=2\columnwidth]{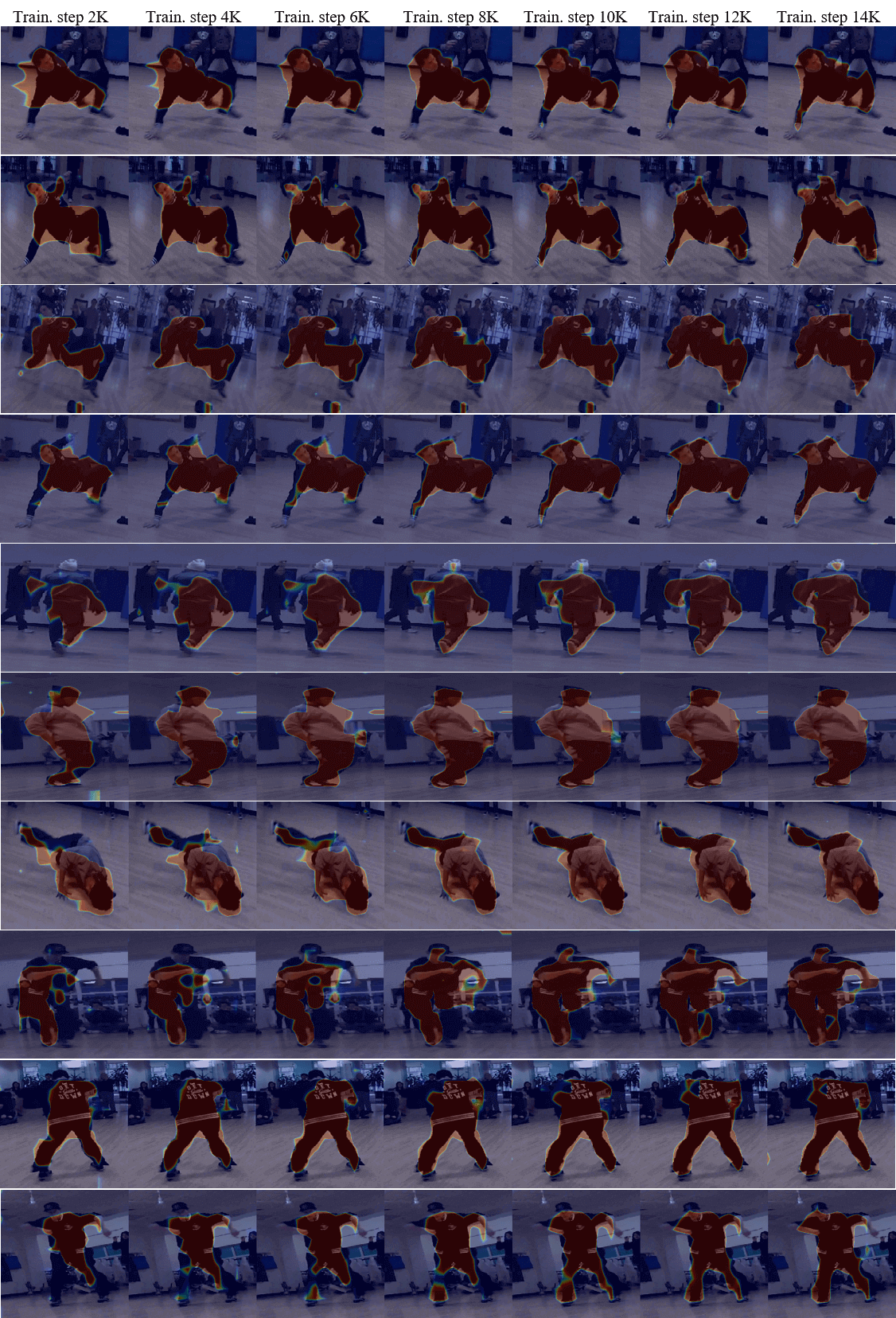}
\caption{Qualitative result of multiview segmentation.}
\label{Fig:compare1}
\end{figure*}

\begin{figure*}
\centering 
\includegraphics[width=2\columnwidth]{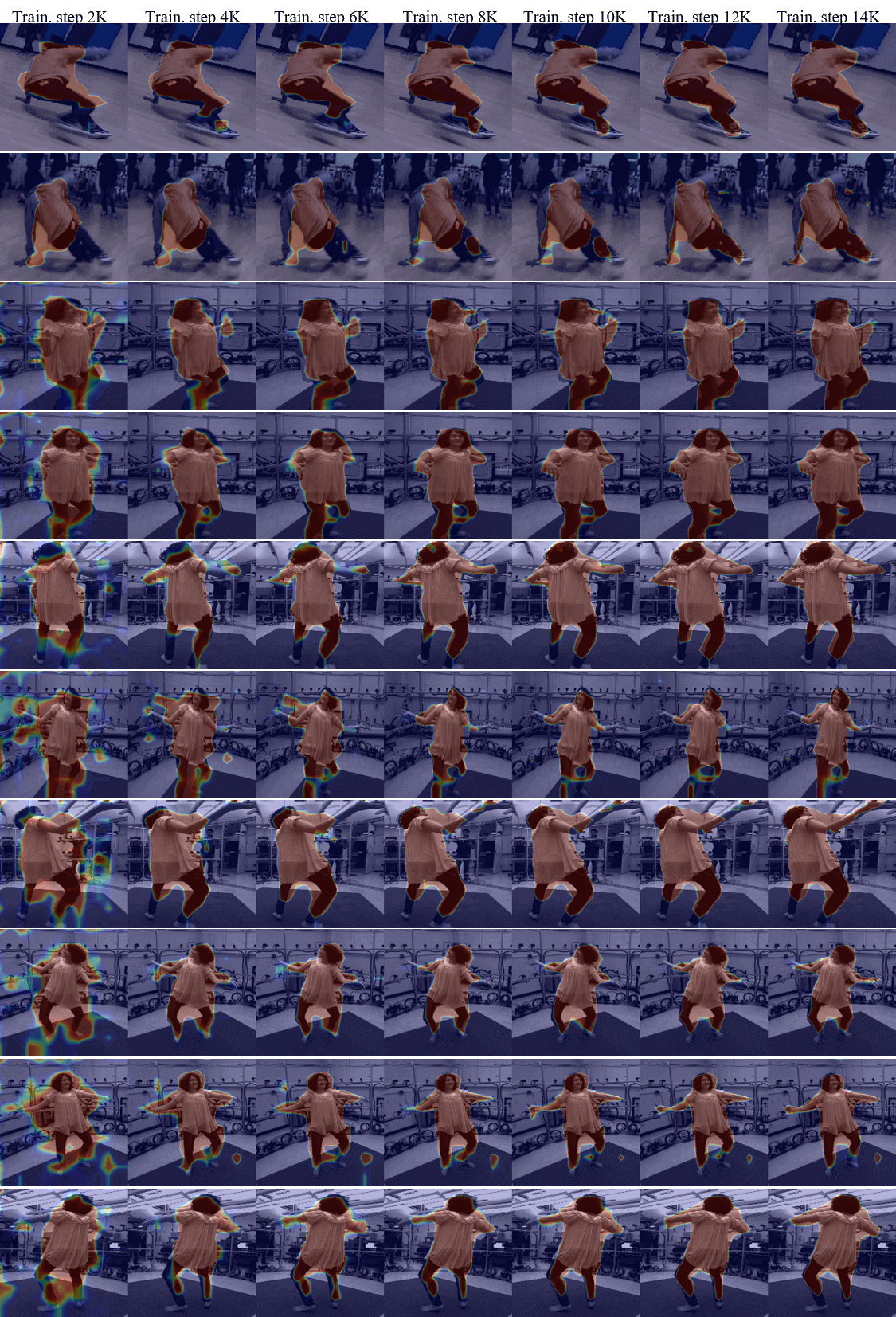}
\caption{Qualitative result of multiview segmentation.}
\label{Fig:compare2}
\end{figure*}

\begin{figure*}
\centering 
\includegraphics[width=2\columnwidth]{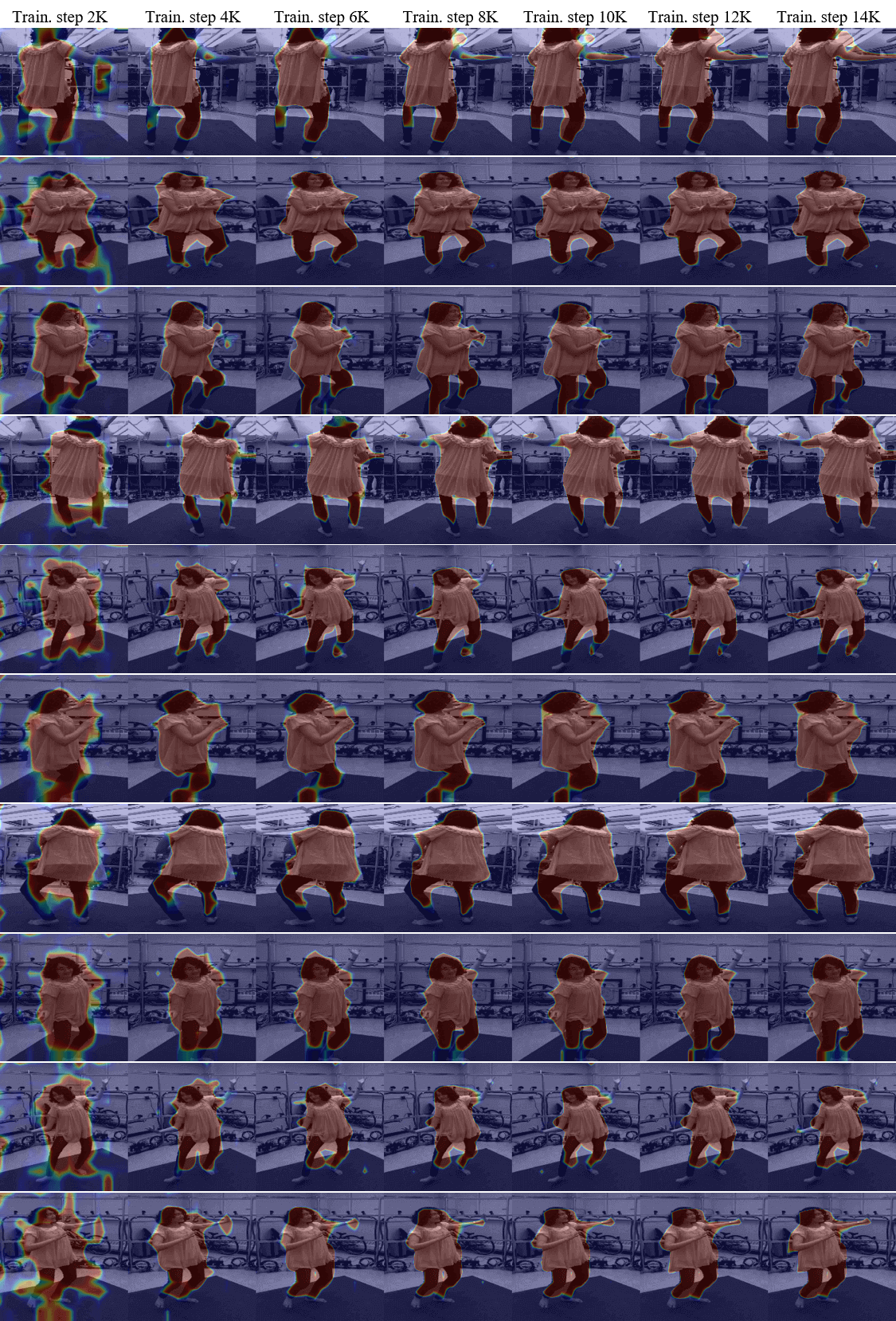}
\caption{Qualitative result of multiview segmentation.}
\label{Fig:compare3}
\end{figure*}

\begin{figure*}
\centering 
\includegraphics[width=2\columnwidth]{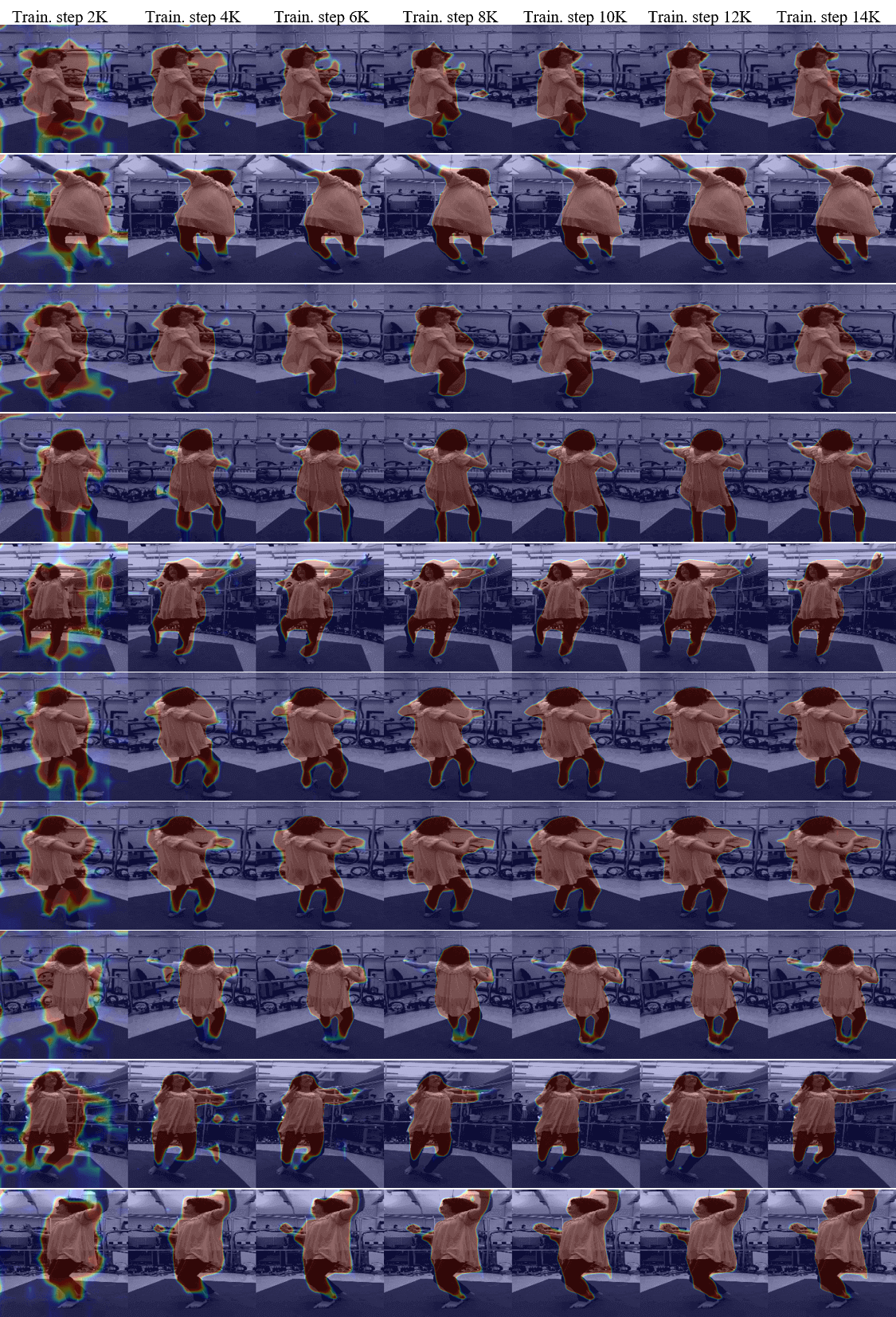}
\caption{Qualitative result of multiview segmentation.}
\label{Fig:compare4}
\end{figure*}

\begin{figure*}
\centering 
\includegraphics[width=2\columnwidth]{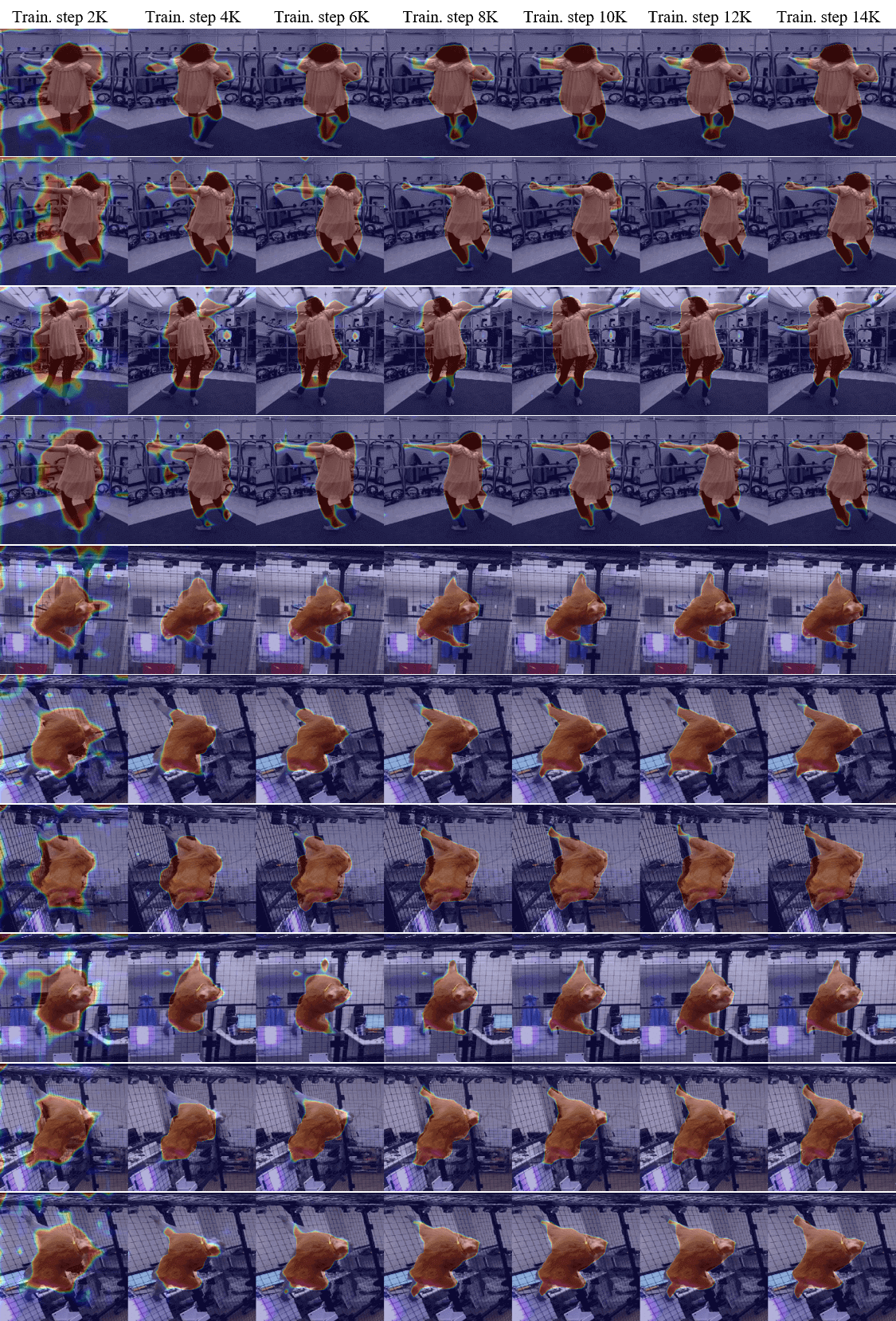}
\caption{Qualitative result of multiview segmentation.}
\label{Fig:compare5}
\end{figure*}

\begin{figure*}
\centering 
\includegraphics[width=2\columnwidth]{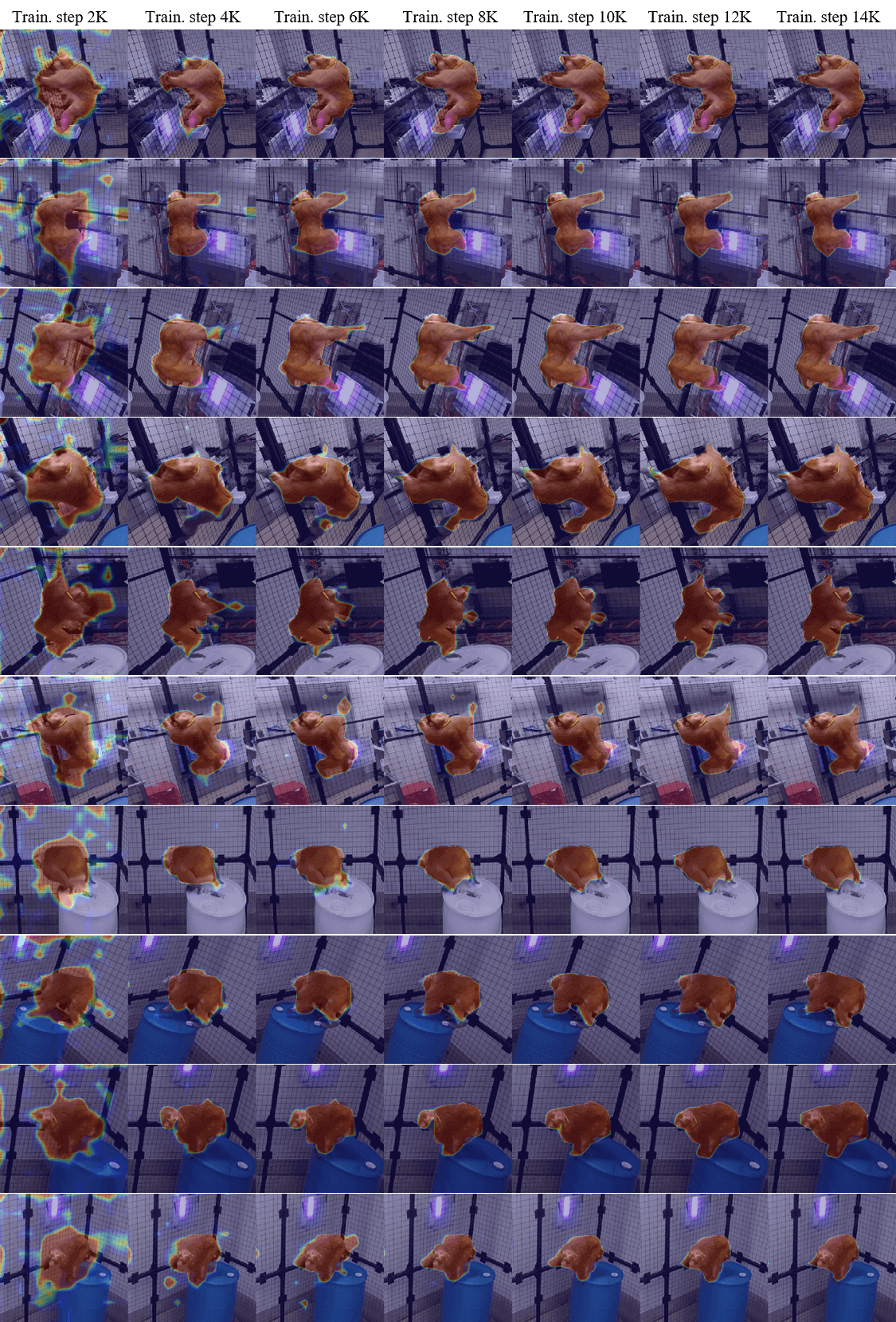}
\caption{Qualitative result of multiview segmentation.}
\label{Fig:compare6}
\end{figure*}

\begin{figure*}
\centering 
\includegraphics[width=2\columnwidth]{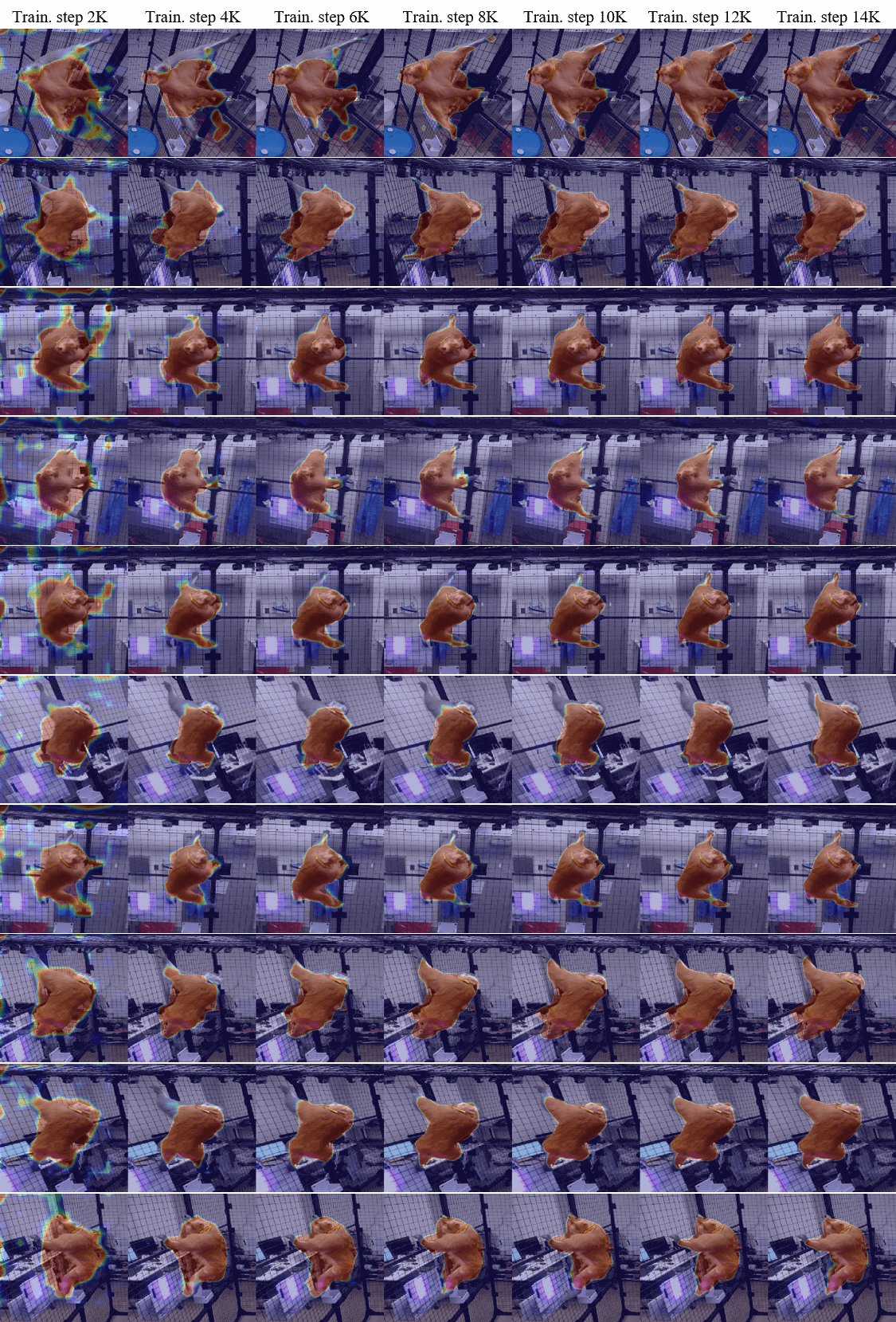}
\caption{Qualitative result of multiview segmentation.}
\label{Fig:compare7}
\end{figure*}

\begin{figure*}
\centering 
\includegraphics[width=2\columnwidth]{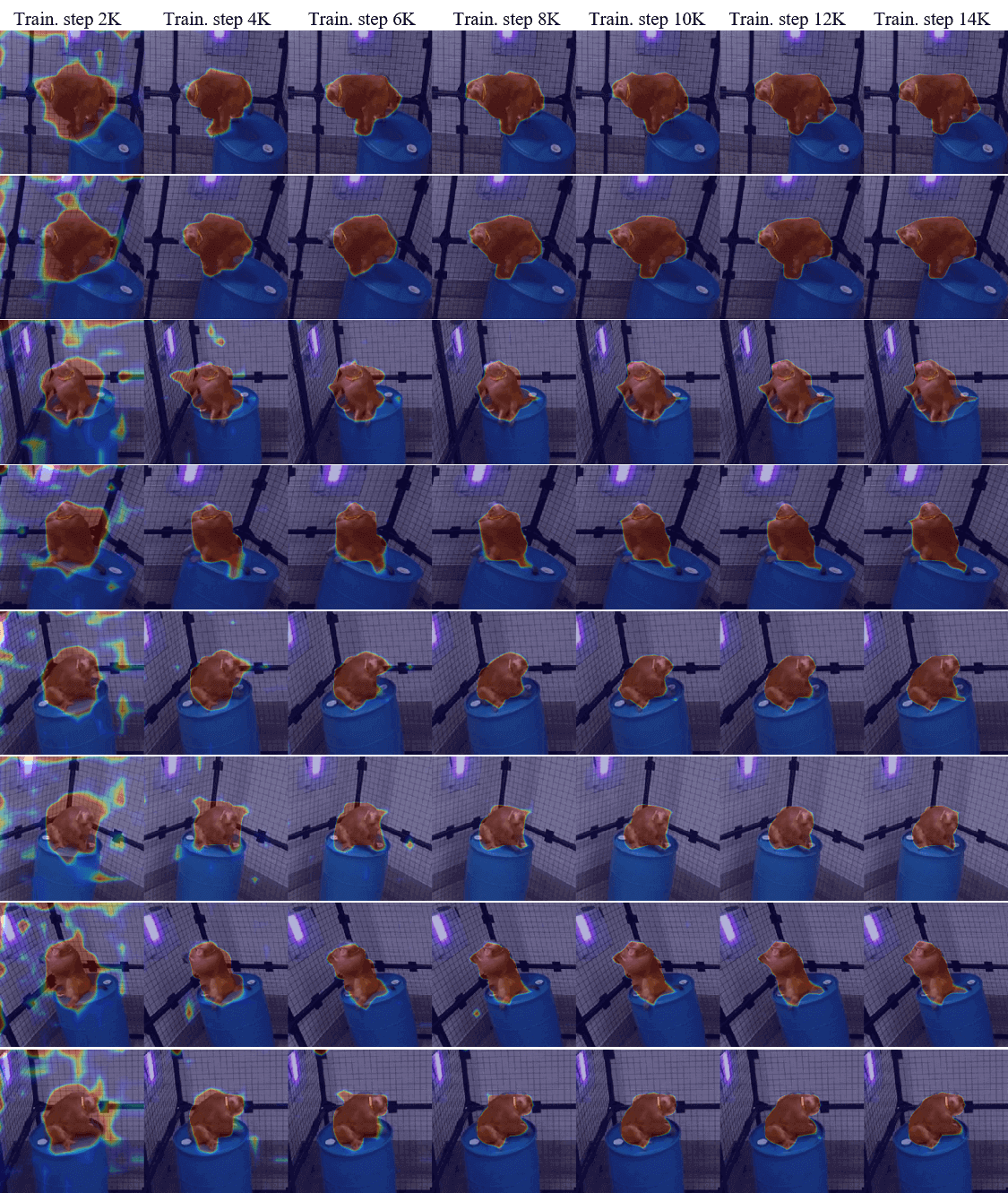}
\caption{Qualitative result of multiview segmentation.}
\label{Fig:compare8}
\end{figure*}

\begin{figure*}
\centering 
\includegraphics[width=2\columnwidth]{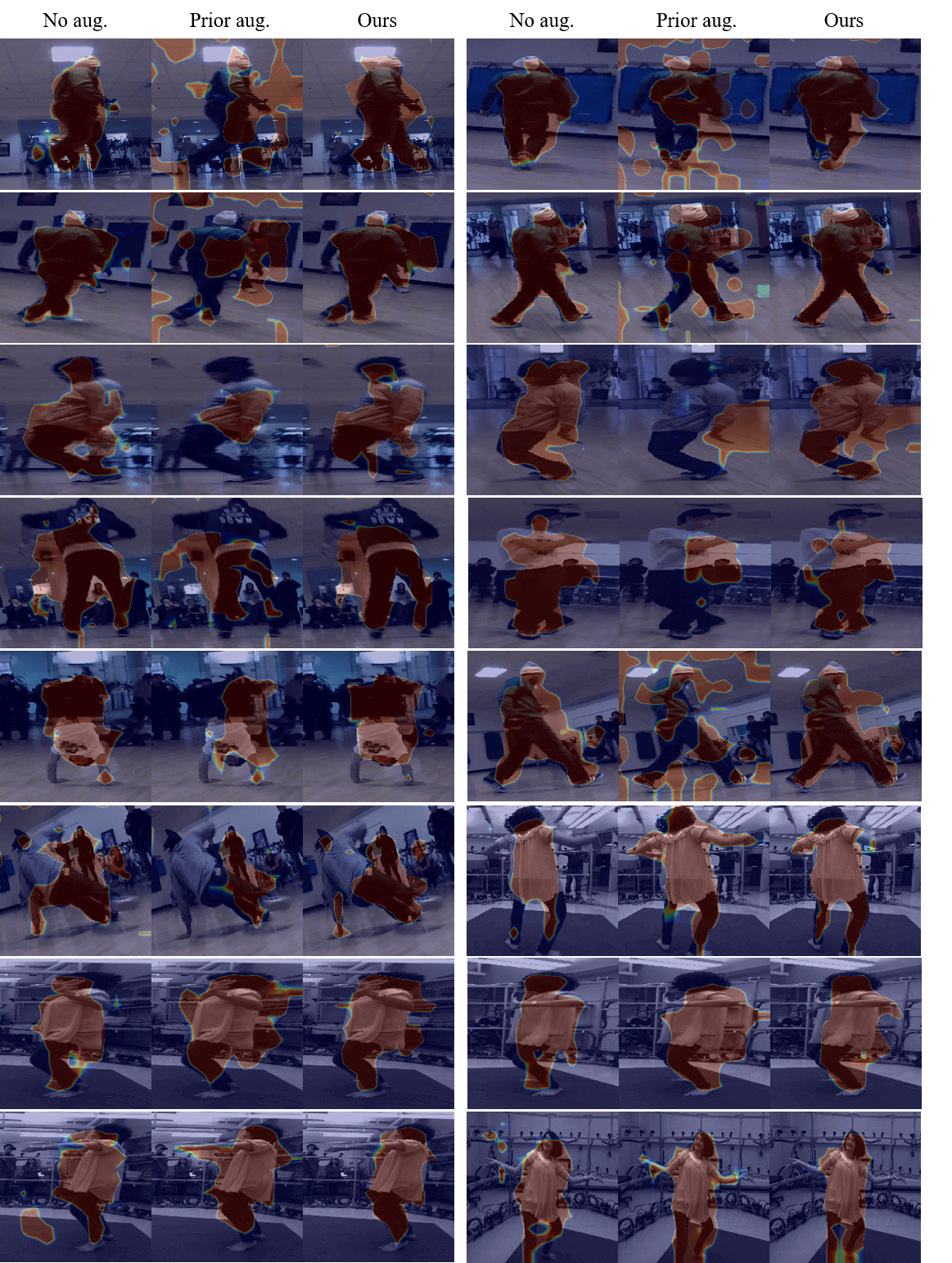}
\caption{Qualitative result of multiview segmentation.}
\label{Fig:b1}
\end{figure*}

\begin{figure*}
\centering 
\includegraphics[width=2\columnwidth]{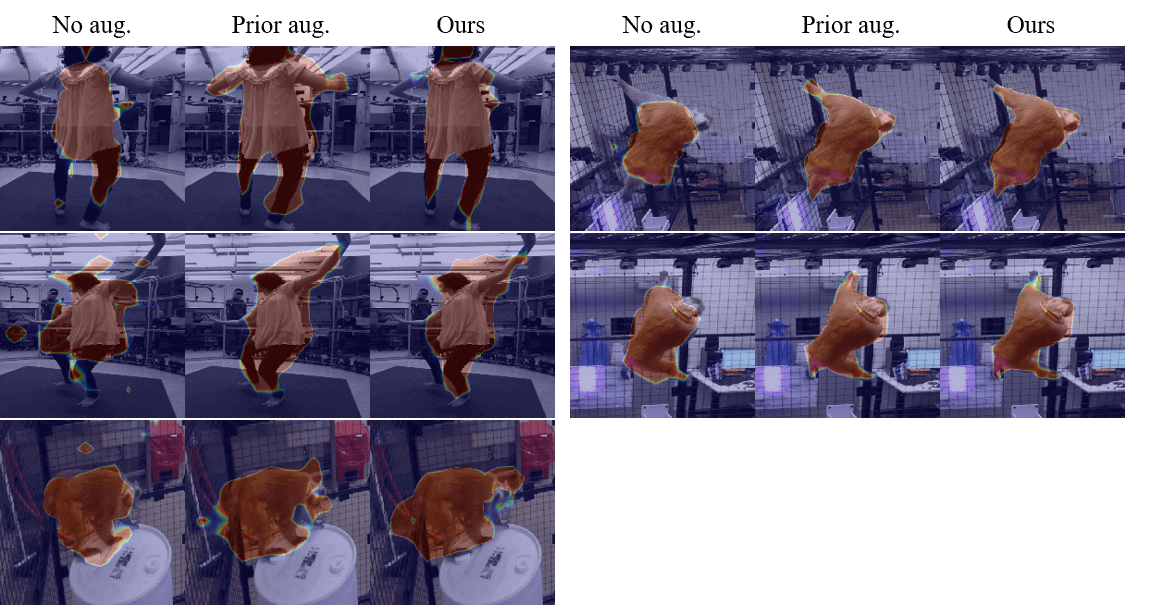}
   \caption{Failure cases.}
\label{Fig:b2}
\end{figure*}

\end{document}